\documentclass{article}

\def\dd{\delta}

\def\AA{\alpha}

\def\DD{\Delta} 
 
\def\G{\gamma }
\def\Ep{\varepsilon}
\def\0{\varnothing}
\def\BB{\beta}
\def\ll{\lambda}
\def\RR{\mathbb R}
\def\NN{\mathbb N}
\def\grad{\nabla}
\def\ds{\displaystyle}
\def\wt{\widetilde}

\def\si{\sigma}

\def\am{\mathrm{argmin}}

\def\Sumd{\sum_{j=1}^d}
\def\Sum{\sum_{i=1}^n}

\def\SumN{\sum_{i=1}^{N}}
\def\Su2mN{\sum_{i=2}^{N}}
\def\Su2m{\sum_{i=2}^{N}}

\def\Ex{\mathbb{E}}

\def\C{\mathcal{C}}
\def\X{\mathcal{X}}
\def\W{\mathcal{W}}

\def\A{\mathcal{A}}
\def\B{\mathcal{B}}
\def\P{\mathcal{P}}

\def\V{\mathcal{V}}
\def\S{\mathcal{S}}

\def\1{{\boldsymbol 1}}

\def\ol{\overline}

\usepackage[rightcaption]{sidecap}
\usepackage{amsthm}
\usepackage{graphicx}
\usepackage{amssymb}
\usepackage{epstopdf}
\usepackage{mathtools}
\usepackage{url}
\usepackage{wrapfig}
\usepackage[linesnumbered,ruled,vlined]{algorithm2e}
\usepackage{caption}
\usepackage{enumitem}
\usepackage[margin=1in ]{geometry}
\usepackage{natbib}
\usepackage{color}

\sidecaptionvpos{figure}{c}
\theoremstyle{definition}
\newtheorem{theorem}{Theorem}
\newtheorem{lemma}{Lemma}
\newtheorem{definition}{Definition}
\newtheorem{example}{Examples}

\begin{document}

\makeatletter
\let\@addpunct\@gobble
\makeatother

\begin{center}
 \bf \Large \noindent Lazy Online Gradient Descent is Universal on Polytopes
\end{center}

\vspace{1cm}
  
  \openup 0.2em
  
  \noindent {\bf Daron Anderson}\hfill   andersd3@tcd.ie \\
  \noindent {\it Department of Computer Science and Statistics}\\
  \noindent {\it Trinity College Dublin }\\
  \noindent {\it Ireland}\\

  \noindent {\bf Douglas Leith} \hfill doug.leith@scss.tcd.ie\\
  \noindent {\it Department of Computer Science and Statistics}\\
  \noindent {\it Trinity College Dublin }\\
  \noindent {\it Ireland}

%\editor{Kevin Murphy and Bernhard Sch{\"o}lkopf}

%\maketitle

\begin{abstract}%   <- trailing '%' for backward compatibility of .sty file
	\openup .5em 
	\noindent We prove the familiar Lazy Online Gradient Descent algorithm is universal on polytope domains. That means it gets $O(1)$ pseudo-regret against i.i.d opponents, while simultaneously achieving the well-known $O(\sqrt N)$ worst-case regret bound. 
	For comparison the bulk of the literature focuses on variants of the Hedge (exponential weights) algorithm on the simplex.
	These can in principle be lifted to general polytopes; however the process is computationally unfeasible for many important classes where the number of vertices grows quickly with the dimension. The lifting procedure also ignores any Euclidean bounds on the cost vectors, and can create extra factors of dimension in the pseudo-regret bound. Gradient Descent is simpler than the handful of purpose-built algorithms for polytopes in the literature, and works in a broader setting. In particular  existing algorithms assume the optimiser is unique, while our bound allows for several optimal vertices. 
	\vspace{2mm} 
\end{abstract}


%\openup .2em

\section{Introduction}

 \noindent The lazy anytime variant of Online Gradient Descent is known to achieve  $O(L_2\sqrt N  ) $ regret.  In this paper we show that for polytope domains the algorithm specialises to get finite  $O(L_2^2/\DD)$ pseudo-regret if the cost vectors turn out to be i.i.d. The new bound is independent of dimension.  To the authors' knowledge the only similar known result is the $O(\log(d) L_\infty^2/\DD )$ bound of \cite{OptimalHedge} for Hedge on the simplex.
 
 Thus Gradient Descent is a computationally efficient way to get small pseudo-regret on polytopes where the number $V$ of vertices is large. For comparison, the naive approach of lifting the problem to the  $V$-simplex (see Section 3) and running Hedge is computationally expensive as we must update vectors of length $V$ each turn. Gradient Descent on the other hand only updates vectors of length $d$ equal to the dimension of the action set.
 %This curse of dimensionality is typical for combinatorial optimisation problems where the polytope is the convex hull of some large finite set  $V \subset \RR^d$. We seek some optimal element of $V$ and interior points are treated as a probabilistic choice of the vertices. 
 
 This is significant as typically $V$ grows quickly with $d$. Even the simple cube has $V=2^d$ vertices. The more complex problem of learning permutations (see Tables 1 and 2) leads to polytopes with   $\Omega(d\,!)$ vertices.  
 Applications of learning permutations include ranking user preferences, travelling salesman problems, and assigning    ``vehicles'' to ``routes'' in a transportation problem. More generally this curse of dimension occurs in combinatorial optimisation, where we seek an optimal action from a large finite set. The nature of the cost function allows us to embed the actions as points in a lower-dimensional space. The polytope domain arises as the convex hull of the embedded actions, with each interior point treated as a probabilistic choice of the vertices. 
 
 The second advantage of Gradient Descent over Hedge is better dimensional scaling (Table 1) of the performance bounds. When $V$ grows quickly with $d$ the $\log$ term in the Hedge bounds can contribute an extra factor of $d^2$ to the pseudo-regret. It is also worth mentioning that the geometry of Gradient Descent makes it better-suited to problems that naturally satisfy Euclidean bounds rather than the $L_\infty$-bounds typical of Hedge problems. Compare Tables 1 and 2.

Earlier lines of research on universal algorithms required the development of complicated algorithms purpose-built to be universal. On the other hand  Gradient Descent and Hedge are simple and predate this line of research. They suggest such complex algorithms are unnecessary. Gradient Descent and Hedge   are popular and widely used. Hence any improved results we get ``for free'' have immediate broad application. 
 
 Our proof uses a novel approach where we follow the sequence of unprojected actions and show the projected actions snap to an optimal vertex with high probability.  
This analysis is not available to Hedge-type algorithms which can only approach the optimal vertex asymptotically.  For existing non-Hedge algorithms, the proofs tend to obscure the geometry by focusing on the regret and using telescoping series, rather than the tracking the actions themselves.   We also use a vector concentration result that seem to be new in this context.

 \subsection{Related Work}
 The term {\it universal algorithm} comes from the bandit setting (\cite{BestofBoth,AnOptimal,OnePractical,MoreAdaptiveAlgorithms,cesa2007improved,luo2015achieving,gaillard2014second,van2015fast,NearlyOptimal,ImprovedAnalysis}).  These are online algorithms designed to achieve optimal regret  in the antagonistic setting, and simultaneously get better performance against easier (for example i.i.d) data sets.

 In the full-information setting the optimal bounds are $O(\sqrt N)$ regret against antagonistic data and  $O(1/\DD)$  pseudo-regret against i.i.d data. See for example \cite{Lowerbound} and \cite{LowerBoundBlog}. Here the  universal algorithms    terminology   is less established. To our knowledge the problem was first studied by \cite{FlipFlop} on the simplex. Their  FlipFlop algorithm interleaves a new variant of Hedge   with Follow-The-Leader (FTL) to get regret bounds  $R_N \le O \big ( \sqrt N \big )$ and  $R_N \le O \big (  R^{\text{FTL}}_N  \big )$ simultaneously. 
 
 \cite{ExploitingEasyData} give a black-box method for general domains to get  $R_N \le O \big (  R_N^\A + \sqrt{N \log N} \big )$ and $R_N \le  \,R_N^\B + O(1) $ for any algorithms $\A,\B$. Here $\A$ should be a worst-case algorithm and $\B$  specialised to some class of  {\it easy data}. For example $\A = $  Gradient Descent and $\B = $ FTL gives $R_N \le O(\log N)$ if the cost functions turn out to be  strongly convex; otherwise we fall back on the worst-case $R_N \le O(\sqrt{N \log N})$  bound.
 
 \cite{gaillard2014second} give  a variant of Prod with separate learning rates for each arm. As a corollary they get  $\Ex[R_N] \le O(1)$ against i.i.d costs with unique optimal arm. Their main theorem is  a new $\wt O(\sqrt N)$ bound where the coefficient depends on the observed regret. The SQUINT and iProd algorithms of \cite{squint}  lower the dimension of this $\wt O(\sqrt N)$ bound for easy data with many {\it sufficiently good} experts. The latter works for polytopes but requires a vertex-decomposition every turn and has the potential to be computationally expensive. 
 The MetaGrad algorithm (\cite{metagrad}) lowers the exponent of the Gailland bound for data with  the $\beta$-Bernstein condition for $0 \le \beta \le 1$. In that case they get  $\Ex[R_N] \le O(N^{\frac{1-\BB}{2-\BB}})$.

 %If we assume the data follows a $\beta$-Bernstein distribution then MetaGrad algorithm \citep{metagrad} improves the Gailland $O(\sqrt N)$ bound to  $O(N^{b})$ where $b$  
 
 %$a \in (0,1/2]$ depending on the Bernstein parameter of the data. Unfortunately the data has a trivial 
 
 %on general domains against i.i.d costs with the so-called Bernstein Condition. This condition involves the cost distribution and shape of the domain. In particular it fails if the domain is a polytope with  expected cost minimised at two separate vertices. 
  
 \cite{FTLBall} consider  more straightforward algorithms. For polytopes and i.i.d cost vectors with unique optimal vertex   they show  FTL gets $\Ex[R_N] \le O(L_\infty ^3 d/r^2)  $. Here $L_\infty$ bounds the $\infty$-norm of the costs and $r$ is the largest distance we can move the expected cost without changing the minimiser.  To get a universal algorithm they use the Prod($\A,\B$) to combine with Gradient Descent and get  $O(\sqrt {N \log N})$ and $O(1)$ bounds.  
 
 The closest work to this paper is \cite{OptimalHedge} which proves the familiar Hedge algorithm  is universal. It simultaneously achieves $O(\sqrt N)$ regret and  $O(1/\DD)$ pseudo-regret bounds with optimal dependence on hyperparameters. Hedge and Gradient Descent both fall under the Follow-the-Regularised-Leader (FTRL) framework; Hedge is FTRL with the entropic regulariser $\sum_j x_j \log x_j$ while our version of Gradient Descent uses the  quadratic regulariser $\frac{2}{\eta \sqrt n}\|y_1 -x\|^2$.  
 Their result is surprising because Hedge is simpler than the above purpose-built methods, and predates the recent interest in universal algorithms. For example see \cite{Kivinen}. 
 
 The above bounds do not apply to our setting, when there are several optimal vertices. In that case we cannot achieve $\Ex[R_N] \le O(1)$  by any algorithm and might indeed have $\Ex[R_N] \ge \Omega(\sqrt N)$  with no   extra assumptions   on the cost vectors.  See Lemma \ref{bigEx1} in Appendix B  or \cite{RootNLowerBound} for a more general analysis. This rules out the $R_N \le  \,R_N^\B + O(1) $ bound of \cite{ExploitingEasyData} since there is no suitable choice of $\B$. The MetaGrad algorithm only gives exponent $\frac{1-\BB}{2-\BB} = \frac{1}{2}$ since for several optimal vertices the Bernstein condition  fails for $\BB>0$.
 
 Our stronger $O(1)$ bound is possible because we focus not on the expected regret but on the pseudo-regret which is a smaller quantity in general. This is standard practice in the i.i.d setting to avoid the impossibility results mentioned above.

\subsection{Summary and Contribution}

 \openup.2em
 
 \noindent Section 2 contains our main result. Theorem 2 says  running   Gradient Descent on a polytope  $\P$   against i.i.d cost vectors gives pseudo-regret  $O(D^2 L^2/\DD)$ independent of dimension. Here $L = \sup \{\|a_n\|: n \in \NN\}$ bounds the Euclidean norm of the cost vectors;  $D = \sup\{\|x-y\|:x,y \in \P\}$ is the diameter of the polytope; and  $\DD$ is the gap between the expected cost of the optimal vertex and the expected cost of the best suboptimal vertex.  
{
 
 Section 3 specialises Theorem 2  to some well-studied classes of polytopes.  
 Table 1 compares our bounds for Gradient Descent to those for lifted Hedge. For the simplex, cube, Birkhoff polytope, and (signed) permutahedron, Gradient Descent scales better with dimension than Hedge. In particular for the permutahedron   Hedge performs worse by a factor of $d$ and $d^2$ in the antagonistic and i.i.d cases, respectively.

Section 4 contains two variants of Theorem 2 that are independent of how the domain is embedded in $\RR^d$. Theorem \ref{projEuclidean} replaces the coefficient $L$ in our main $O(D^2 L^2/\DD)$ bound with the potentially smaller quantity obtained by deleting the component of the cost vectors perpendicular to the domain. Theorem \ref{T3} replaces $L$ with the quantity $L^2_\infty/ W^2$ where  $L_\infty = \sup\{|a_n \cdot(x-y)|:x,y \in \P\}$ bounds the costs intrinsically; and the width $W$ (see Definition \ref{widthdef}) is the smallest number $w$ such that $\P$ is contained between two hyperplanes of distance $w$ apart.  %Theorem \ref{18}  generalises the standard $\infty$-norm bound  for problems on the simplex.

Section  5 considers the examples from Section 3 under the intrinsic bounds above.\footnote{\openup 0.2em To the authors' knowledge the widths of the polytopes in Table  2 do not appear elsewhere in the literature. Computing the widths is nontrivial, and we use   a probabilistic counting trick famously attributed to Paul   Erd\H{o}s  (\cite{ErdosCounting})  and suggested by David E \cite{Wide}.
In fact we could not find a modern treatment of the width of the simplex. See Appendix A.}
 For the  permutahedron  Gradient Descent performs slightly better than Hedge. However for all examples other than the simplex Hedge quickly becomes unfeasible.  
 Section 6 discusses the computational cost of Gradient Descent, in particular computing the projection. We also mention some  open problems and possible improvements to our analysis.

\section*{Terminology and Problem Setup}

\noindent %Throughout we write   $b_1,b_2,\ldots \in \RR^d$ for the  cost vectors. One turn $n$ we know $b_1,b_2,\ldots, b_{n-1}$ and must select an {\it action } in a given convex set $\X \subset \RR^d$ called the domain. The goal is to minimise the regret  $\SumN a_j \cdot (x_j- x^*)$ for the best fixed action $x^* \in \am \left \{ \SumN a_j \cdot x : x \in \X \right\}$  in hindsight. If the cost vectors are i.i.d we write them as $a_1,a_2,\ldots$ and write $a = \Ex[a_n]$. In this case the goal is to bound the pseudo-regret $\SumN a_j \cdot (x_j- x^*)$ for $x^* \in \am \{a \cdot z \}$. 
Throughout $d$ is the dimension of the online optimisation problem. The {\it cost vectors} $a_1,a_2,\ldots \in \RR^d$ are realisations of a sequence of i.i.d random variables with each $\Ex[a_i]=a$. When we write $b_1,b_2,\ldots $ for the cost vectors we make no assumptions on whether they are i.i.d or otherwise. Unless otherwise specified we assume bounds of the form $\|a_i - a\| \le R$ and $\|a_i\| \le L$ for $\|\cdot \|$ the Euclidean norm.

In the problem setup the {\it domain} or {\it action set}  $\X \subset \RR^d$ is compact and convex. On each turn $n$ we know $b_1,b_2,\ldots, b_{n-1}$ and must select an action $x_n \in \X$. In the antagonistic setting our goal is to get small  {\it regret}  $\SumN a_i \cdot (x_i- y^*)$ for the best fixed action $y^* \in \am \big \{ \SumN a_i \cdot x : x \in \X \big\}$  in hindsight. In the i.i.d setting our goal is to get small {\it pseudo-regret} $\SumN a \cdot (x_i- x^*)$ for the expected minimiser $x^* \in \am \{a \cdot x: x \in \X \}$. 

Each algorithm for online linear optimisation extends to the more general setup of online convex optimisation. Given  convex cost functions $f_1,f_2,\ldots$ we can run the algorithm on  cost vectors $b_n = \grad f_n(x_n)$ and use convexity to bound the regret $\SumN \big( f_i(x_i) -f_i(y^*) \big) \le \SumN \grad f_i(x_i) \cdot(x_i-y^*) =  \SumN b_i \cdot(x_i-y^*)$ which can be controlled using a linear algorithm.

%We write $\S_d$ for the $d$-simplex $ \{x \in \RR^d: \mbox{all } x(j) \ge 0  \mbox{ and } x(1) + \ldots + x(d) =1\}$.  
By an {\it affine subspace} of $\RR^d$ we mean a translation of a vector subspace. The {\it affine hull} of $A \subset \RR^d$ is the set  $\big \{\!\sum_{i=1}^k \AA_i x_i:  x_i \in A ,\AA_i \in \RR, \sum_{i=1}^k \AA_i  =1 \big \}$. This is the smallest affine subspace containing $A$. The corresponding linear subspace    $\big \{\!\sum_{i=1}^k \AA_i x_i:  x_i \in A ,\AA_i \in \RR, \sum_{i=1}^k \AA_i  =0 \big \}$ is called the {\it direction} of $A$. For convex $A$ the affine hull and direction can be written $\{x + t(y-x): x,y \in A, t \in \RR\}$ and $\{t(y-x): x,y \in A, t \in \RR\}$ respectively. The {\it dimension} of an affine subspace is the dimension of the corresponding vector subspace. The dimension of a polytope is the dimension of its affine hull.  %or equivalently the   set $\{\AA (x-y): x,y \in A,\AA \ge 0, \}$ or equivalently  
%The affine hull is a translation of a vector-subspace of $\RR^d$ and is in fact the smallest set containing 
 
 \subsection*{Lazy Online Gradient Descent}

 Online Gradient Descent is among the simplest and most familiar algorithms for online linear optimisation.    For the original proof that Gradient Descent has $O(\sqrt N)$ regret see \cite{Z}. For a modern exposition see Chapter 2 of \cite{Purple1}. For a self-contained proof of the anytime case see Appendix E. 

\begin{theorem} \normalfont \normalfont \label{worstcase}  
	Given cost vectors $b_1.b_2,\ldots, b_N$ with all $\|b_i\| \le L$ Algorithm 1 with parameter $\eta$ has regret satisfying 
	\begin{align*}
	\sum_{i=1}^N b_i \cdot (x_{i} - y^*) \le  LD+ \left (\frac{\|\X\|^2 }{2 \eta} +  2\eta L^2 \right)  \sqrt {N} \end{align*} 
	for $\|\X\|=  \max\{\|x -y_1\|: x \in \X\}$ and $D = \max\{\|x-y\|: x,y \in \X\}$ the diameter of $\X$. In particular for $y_1 \in \X$ and $\eta = \|\X\|/2L$ we have
	\begin{align*}
	\sum_{i=1}^N b_i \cdot (x_{i} - y^*) \le    LD+ 2L \|\X\|  \sqrt N \le 3LD \sqrt N.\end{align*}  
\end{theorem}

\begin{algorithm}[!t]\caption{Lazy Anytime Online Gradient Descent}
	\DontPrintSemicolon % Some LaTeX compilers require you to use \dontprintsemicolon instead
	\KwData{Action set $\X \subset \RR^d$. Base point $y_1 \in \RR^d$. Parameter $\eta > 0$. Cost vectors $b_1,b_2,\ldots \in \RR^d$.\;}
	%Regularisation function $R(w) = \sum_{i=0}^d w[i] \log w[i]$ on $\W$ 
	\text{select action} $\ds x_1 = \Pi_\X(y_1)$\;
	\text{pay cost} $b_1 \cdot x_1$\; 
	\SetKwBlock{Loop}{Loop}{EndLoop}
	\For{$n=2,3,\ldots$}{
	
	\text{recieve} $b_{n-1}$\;
	$\ds y_n = y_1 -\eta \left( \frac{b_1 + \ldots + b_{n-1}}{\sqrt {n-1}} \right )$\;
	\text{select action} $\ds x_n = \Pi_\X(y_n)$\;
	\text{pay cost} $b_n \cdot x_n$	
}
\end{algorithm}
%The pseudo-code is shown as Algorithm 1. 
The {\it lazy} terminology comes from \cite{Z} and refers to how  the action $x_n = \Pi_\X(y_n)$ is computed using  a single projection. For comparison so-called greedy variants define the actions iteratively, for example the  action  $x_{n+1} = \Pi_X(x_{n} - b_{n}/\sqrt{n}) $ requires $n$ projections to compute. The lazy aspect of the algorithm is important since lazy and greedy variants are known to behave differently. See \cite{GDStronglyCurved} Section 5. 

%Hence $\am\{a \cdot x: x \in \X\} = \am\{ a \cdot e_j: j \le d\}$. We write $\|\cdot\|$ for the Euclidean norm and for any convex $\Omega \subset \RR^d$ we write $P_\X(x) = \am\{\|y-x\|^2: y \in \Omega\}$ for the projection of $x$ onto $\Omega$. 

\section{Lazy Online Gradient Descent on Polytopes}

\noindent In this section we prove our main result that  Online Gradient Descent on a polytope has  pseudo-regret $O(D^2L^2/\DD)$ in the i.i.d setting.
Henceforth the domain $\P$ is a polytope. That means the convex hull of a  finite set $\V \subset \RR^d$ such that no element is in the convex hull of the other elements. These elements are called the {\it vertices} and are uniquely defined (see \cite{ConvexNotes} Theorem 4.7).  Equivalently  every  polytope is the solution to a finite set of affine inequalities that correspond to the facets of the polytope (see  Gallier Section 4).
%Theorem 4.7 of \cite{ConvexNotes}  says the vertices of $\P$  are exactly the extreme points. Here   an extreme point  $x \in \P$ is one such that there are no $\ll \in (0,1)$ and  $y,z \in \P$  with $y,z \ne x$ and $x = \ll  y + (\ll -1)z$. In other words $x$ is not properly between any other two points of $\P$. In Section 4 Gallier proves the well-known separate 
%Each polytope 
%Equivalently every polytope is the solution to a set of affine inequalities corresponding to the facets. See Section 4 of \cite{ConvexNotes}.
Here a {\it face} is the intersection of $\P$ with a  tangent plane, and a {\it facet} is a face whose affine hull has dimension one less than the polytope itself.
Write $D = \max\{\|x-y\|: x,y \in \P\}$ for the diameter and $\|\P\| = \max\{\|x-y_1\|: x  \in \P\}$ for the radius relative to the basepoint $y_1$ in Algorithm 1.

Write $\V^* =   \am\{a \cdot x : x \in \V\}$ for the  optimal vertices.
Since  every linear function on a polytope is minimised on a vertex $\V^*$ is nonempty.   The vertex set is laminated by the expected cost; write the distinct suboptimality gaps as $\{a \cdot(v - v^*) : v^*  \in \V^*, v  \in  \V -\V^*\} = \{\DD_2,\ldots, \DD_U\}$ for some $U \le |\V|$ and $\DD=\DD_2  < \ldots < \DD_U$. The layers are  $\V_j = \{v \in \V: a \cdot(v - v^*) = \DD_j \, \forall v^* \in \V^*\}$. We abuse notation and also write $\DD_v = a \cdot(v - v^*)$ for each $v^* \in \V^*$.

%Let $\mathcal C$ be the {\it Chebychev radius}. That means the radius of the smallest closed ball that contains $\P$. The centre of that sphere is called the {\it Chebychev center} of $\P$.
%Write $\V^* =   \am\{a \cdot x : x \in \V\}$ for the set of vertices where $a = \Ex[a_n]$ is minimised.  

\begin{theorem} \normalfont\label{T2}   Let $\P \subset \RR^d$ be a polytope. Suppose the cost vectors $a_1,a_2,\ldots$ are i.i.d with all $\|a_i\| \le L$ and $\|a_i- a\| \le R$ for $\Ex[a_j] = a$. Suppose we run Algorithm 1 with domain $\P$ and parameter $\eta>0$ and starting point $y_1 \in \RR^d$.  Then for each $\AA > 3$  and $\BB = \frac{1}{3}- \frac{1}{\AA}$ the pseudo-regret satisfies
	\begin{align}
	\Ex \left [\sum_{i=1}^{\infty} a \cdot( x_i-v^*)   \right ] \le LD +\frac{1}{\DD}\left ( \frac{\|\P\|^2}{2 \eta}  + 2 \eta L^2 +  \sqrt{ \frac{\pi}{2} }RD \right ) \left( \frac{3}{2} \frac{ \AA D^2 }{ \eta  }      + \frac{\eta  L^2 }{\AA   }   \right )\ \, \label{firstbound}\\ +   \frac{12 R^2D^2}{\BB^2 \DD}  \exp \left (-\frac{1}{2}\left(\frac{\AA \BB \|\P\|}{\eta  R } \right)^2 \right).\notag
	\end{align}

	for $\|\P\| = \sup \{\|x-y_1\|: x \in \P\}$ and $D = \sup\{\|x-y\|: x,y \in \P\}$.
	%In particular for $y_1 \in \P$ we have $\|\P\| = \C$ is the Chebychev radius and for $\eta = \C/L$  we have 
	In particular for  $y_1 \in \P$ and $\eta = D/2L$  and each $v^* \in \arg \! \min \{a \cdot x: x \in \P\}$ we have 
	 
	\begin{align}
	\Ex \left [\sum_{i=1}^{\infty}   a \cdot( x_i-v^*)   \right ]  \le   L D + \left( 31 L\left (  2L  + \! \sqrt{ \frac{\pi}{2} }R \! \right )      \!+ \!    15 R^2\right) \frac{D^2}{\DD}    = O \left( \frac{L^2 D^2}{\DD}\right).\label{secondbound}
	\end{align}
	
	%\begin{align*}
	% 5LD\left( 48\frac{L D }{  \DD }    + \frac{ \DD}{32 LD  }\right ) + \frac{R^2 D^2}{40} \left( \frac{1}{\DD_d } +    \frac{2 }{\DD}  \right) 
	%\end{align*}
	
\end{theorem}

%Note the $L,R$ are Euclidean bounds for the cost vectors. In Section XXX we  give more natural bounds for the polytope in question. %and allow for more general basepoints.

Theorem  \ref{T2} is proved using   several lemmas. To state the lemmas  we recall some terminology.  For convex $\X \subset \RR^d$ and $y \in \RR^d$ we write $\Pi_\X(y)= \am\{\|x-y\| : x \in \X\}$ for the Euclidean projection onto $\X$. %Convexity of $\X$ ensures the minimiser is unique.  
The normal cone to $\X$ at $x \in \X$ is the set $N_\X(x)   = \{ u \in \RR^d: u \cdot y \le u \cdot x \text{ for all } y \in \X\}$.  
%Note for $x \in \X$ in the interior the normal cone is empty. 
%For smooth $\X$ the normal cone is one-dimensional. 
For $x$ a vertex the normal cone has dimension $d-1$;  for $x$ in the interior of a facet the normal cone has dimension $1$; and in general the dimension of $N_\P(x)$  depends on the dimension of the largest face with $x$ in its interior.
%has dimension  $d-1$ for $x$ a vertex; is one-dimensional for $x$ in the interior of a facet; and in general has  dimension $d-r$ for $r$ the dimension of the largest face with $x$ in its interior. 
For any $x \in \X$ and $u \in N_\X(x)$ we write $T_\X^u(x) = \{y \in \RR^d: u \cdot y = u \cdot x \}$ for the tangent plane to $\X$ at $x$ in the $u$-direction.

%\noindent To begin the proof define the error terms

%$$\Ep_n = \frac{\sum_{i=1}^{n-1} (a_j-a)}{\sqrt {n-1}}.$$

%\noindent For each $x^* \in \am \{a \cdot x: x \in \P\}$ the plane $T = \{a \cdot (x-x^*)=0: x \in \RR^d\}$ is tangent to $\P$ at $x^*$.  The two lemmas will be important. %By rotating and rescaling $\P$ we can assume $a = (1,0,\ldots, 0)$ and $v^* =(1,0\ldots,0)$ and $T = \{ x \in \RR^d: x_1 = 1\}$. 
 Lemma \ref{L1}  follows from the definition of the normal cone.

%\begin{lemma}\label{L3}Let $\X \subset \RR^d$ be convex   with $v \in \X$. The set $\ol N(v) = \{u \in N(v): \|u\|=1\}$ of normal directions at $v$ is compact.
%\end{lemma}

\begin{lemma}\label{L1} \normalfont Let $\X \subset \RR^d$ be convex  with $x \in \X$ and  $-u \in N_\X(x)$. The tangent plane $T^{-u}_\X(x) $   satisfies $T^{-u}_\X(x)  \cap \X = \{y \in \X: -u \in N_\X(y)\} = \am\{u \cdot y : y \in \X\}.$
\end{lemma} 

From Lemma \ref{L1} we see   $-a \in N_\P(v)$ if and only if $v$ is optimal. For $v$ suboptimal 
Lemma \ref{conebound} gives a lower bound for  the angle between $-a$ and $N_\P(v)$.  To interpret the lemma recall the quotient $\frac{a \cdot v}{\|a\| \|v\|}$ is the cosine of the angle between vectors $a$ and $v$. Lemma \ref{conebound} is proved in Appendix C.

%By definition $-a$ is in the normal cone at each optimal vertex. Lemma \ref{conebound} says this holds only for optimal vertices. To interpret the lemma recall the quotient $\frac{a \cdot v}{\|a\| \|v\|}$ is the cosine of the angle between vectors $a$ and $v$. Hence for suboptimal $v$ we get a lower bound between the angle of  $-a$ and $N_\P(v)$.  Lemma \ref{conebound} is proved in Appendix C.

\begin{lemma} \normalfont\label{conebound} For each $v \in \V -\V^*$ we have  $\ds  \inf \left \{ \frac{a \cdot u}{ \|a\|\|u\|} : u \in N_\P(v)\right \} \ge \theta_v $ where we define  \begin{align}\theta_v = \frac{1/2}{1+D^2\|a\|^2/\DD_v^2}-1   \label{angles}                                                                                                                                      \end{align}
 Hence the quantities $\phi_v = \theta_v+1$ are positive.  
\end{lemma}

 Lemma \ref{face}  is proved in Appendix C. 

\begin{lemma} \normalfont\label{face} Each  face $F$ of $\P$ is the convex hull of $F \cap \V$.
\end{lemma}

%\noindent  

%\begin{wrapfigure}{l}{0.6\textwidth} \vspace{-8mm}
%	\begin{center}
%		\includegraphics[width=0.6\textwidth]{PolytopeCropped1}
%	\end{center}
%	\caption{\openup .3em  The polytope has $4$ vertices and the vector $a$ points upwards. The cones $\V_4,\V_3,\V_2,\V_1$ are shown in red, orange, blue and green respectively. The dotted lines indicate the normal cones at each vertex.}
%\vspace{- 4mm}\end{wrapfigure}
  
  \subsection*{Proof Outline}

\noindent The picture to keep in mind throughout the proof is a polytope with a single optimal vertex. The ray from the optimiser in the $-a$ direction is contained in the normal cone at the optimiser. %Hence every point on the ray projects back to the optimiser.
Since the cone is linear, points distance $t$ along the ray are distance $\Omega(t)$ from the interior boundary. Rescaling,  we see that a sequence of points   $\Omega(\sqrt n)$ along the ray can be perturbed by   $O(1)$ and still have all but finitely many points remain in the cone. 
To apply this intuition to Algorithm 1  consider the unprojected actions

\begin{align}y_{n+1}    =   y_1 -  \frac{\eta}{\sqrt{n}}\sum_{i=1}^{n}  a_i  =  y_1 -  \eta \sqrt n a    + \frac{\eta}{\sqrt{n}}\sum_{i=1}^{n} ( a-a_i).  \label{threeterms0} 
\end{align}
To obtain the unprojected action, we start at $y_1$ and move distance $\Omega( \sqrt n) $ along the ray in the $-a$ direction, and then apply the i.i.d perturbation $\Ep_{n+1} = \frac{1}{\sqrt{n}}\sum_{i=1}^{n} ( a-a_i)$.  For  $v^*$ the optimiser and $x_{n+1}$ the action we can rearrange to get
\begin{align} y_{n+1}  -x_{n+1}   =    \underbrace{\phantom{\big (}v^* -  \eta \sqrt n a \phantom{\big (}}_{\text{position on ray}} \ + \ \underbrace{\phantom{\big (} y_1 - x_{n+1} -v^*\phantom{\big (} }_{O(1) \text{ perturbation} }  + \underbrace{ \eta \Ep_{n+1}}_{O(1) \text{ w.h.p}}.\label{threeterms}
\end{align} 
The right-hand-side is a perturbation of the ray from the optimiser. For $y_1 \in \P$ the first part of the perturbation is bounded by the size of the domain. The second part of the perturbation is $O(1)$ with high probabilty  as  a mean zero i.i.d sum. For large $n$ it  follows $y_{n+1} -x_{n+1}$ is normal at the optimiser. Since  $x_{n+1} = \Pi_\P(y_{n+1})$ we   also know $y_{n+1} -x_{n+1}$ is normal at $x_{n+1}$.  Hence Lemma \ref{L1} says the action $x_{n+1}$ is the optimiser and the pseudo-regret is zero on that turn.

 \begin{figure}[h]
  \begin{center}
   \includegraphics[width=0.98\textwidth]{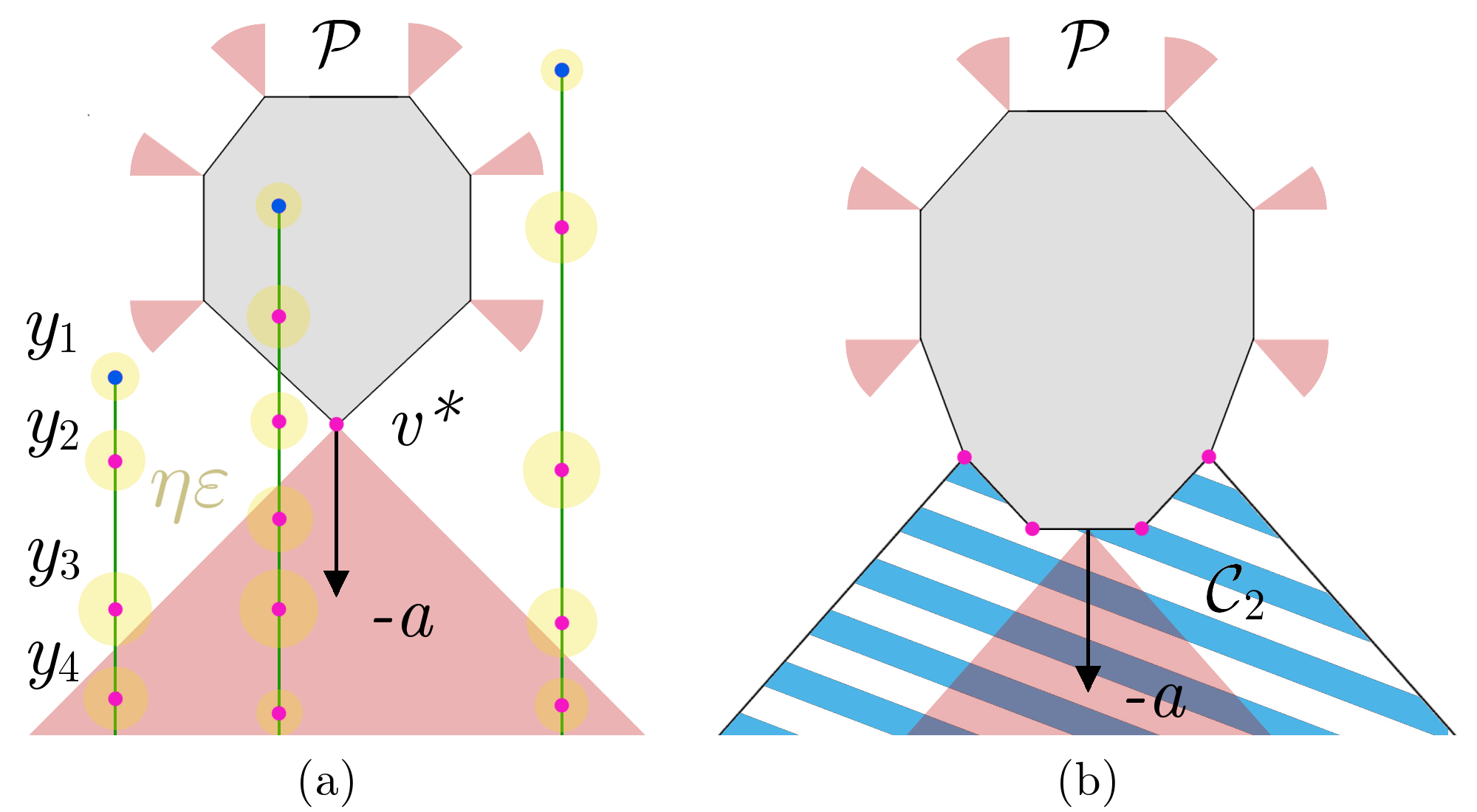}
  \end{center}
  \caption*{\openup 0.2em Figure (a): Schematic of \eqref{threeterms0} for several choices of starting point $y_1$ (blue). The i.i.d perturbation is shown in yellow. The unprojected actions (pink) are eventually inside the normal cone (red) at the optimiser. This cone coincides with the region that projects to the optimiser. 
  
  Figure (b): Schematic of the case for several optimal vertices. Here  $\ol \V_2$ has four elements. The blue area that projects onto $\ol \V_2$ does not coincide with the red cone $\C_2$ and there is no natural place to put the apex. This is not a problem since moving the apex is equivalent to adding another $O(1)$ perturbation to the sequence.} 
 \end{figure} 
The full proof follows the same principle in greater generality. Rather than considering only optimality of $x_n$ we consider the lamination $\V_j = \{v \!\in\! \V: a \cdot(v \!- \!v^*) = \DD_j \, \forall v^* \!\in \!\V^*\}$   and write $\ol \V_j = \V_1 \cup \ldots \cup \V_j$. %Define $\theta_j$ and $\phi_j$ for $j \in \{2,\ldots d\}$ by taking $\DD_v = \DD_j$ in \eqref{angles}.  
For the nested sequence of cones  $\C_j = \bigcup \{N_\P(v): v \in \ol \V_j\}$ it can be shown that if $y_n -x_n \in \C_j$ then the pseudo-regret for that turn is at most $\DD_{j-1}.$ Similar to the previous paragraph, the proof strategy is to  derive conditions on $n$ and i.i.d sum in \eqref{threeterms} to make this happen.

\begin{lemma}\normalfont \label{assumptions}Let  $v \in \V - \V^*$ and $\AA> 1$ be arbitrary. Define $ \BB = \frac{1}{3} - \frac{1}{\AA}$ and suppose 
\begin{align} n > \left ( \frac{ \AA \|\P\| }{\eta \|a\| }\right )^2 \max \left \{1, \frac{1  }{  2\phi_v  }\right\} \qquad \qquad \|\Ep_{n+1}\|  < \BB \sqrt n \|a\| \min \left \{1  ,  \sqrt {2\phi_v} \right \}.\label{smallnoise}\end{align}
Then for $w  = y_{n+1}-x_{n+1}$ we have $v \notin T^{w}_\P(x_{n+1}) $
%Then $v$ is not in the tangent plane at $x_{n+1}$ in the $y_{n+1}-x_{n+1}$ direction.

%$v \notin Q_n \cap \P$ for the tangent plane $$Q_n = \{x \in \RR^d: (y_n - x_n) \cdot x =(y_n - x_n) \cdot x_n\}.$$
\end{lemma}

\begin{proof}For ease of notation write $x,y$ and $\Ep$ instead of $x_{n+1},y_{n+1}$ and $\Ep_{n+1}$. Recall $x$ is the projection of  $y =y_1-\eta\sqrt n a + \eta\Ep$ onto $\P$. To see $y \notin \P$  we claim $\|y_1 - y\|   \ge \|\P\|$. To that end write  $\|y_1 - y\| =  \eta \|\sqrt n a - \Ep\| \ge \eta  (\sqrt n \|a\| - \|\Ep\|)  \ge \eta ( \sqrt n \|a\| - \BB  \sqrt n \|a\|) = (1-\BB) \eta \sqrt n \|a\|$ where we have used the assumption on $\|\Ep_{n+1}\|$. The assumption on $n$  gives $(1-\BB) \eta \sqrt n \|a\| \ge (1-\BB) \eta \frac{\AA \|\P\|}{\eta \|a\|} \|a\| = (1-\BB) \AA  \|\P\| =   \|\P\|$ by definition of $\BB$.

%\begin{align*}
%	\|y_1 - y\| =  \eta \|\sqrt n a - \Ep\| \ge \eta \sqrt n \|a\| - \eta \|\Ep\| \ge  (1-\BB) \eta \sqrt n \|a\| \ge (1-\BB)\AA \|\P\| = \left(2-\sqrt 2 + \frac{1}{\AA}\right )  \AA \|\P\| 
%	\end{align*}

	%\begin{align*}
	%\|y_1 - y\| =  \eta \|\sqrt n a - \Ep\| \ge \eta \sqrt n \|a\| - \eta \|\Ep\| \ge  (1-\BB) \eta \sqrt n \|a\| \ge (1-\BB)\AA \|\P\| = \left(2-\sqrt 2 + \frac{1}{\AA}\right )  \AA \|\P\| 
	%\end{align*}
	Hence $y \notin \P$ and  $y -x \in N_\P(x)$  and  $T^w_\P(x) = \{z \in \RR^d: (y -x) \cdot z = (y -x) \cdot x \}$ is the tangent plane.	For a contradiction suppose  $v \in T^w_\P(x) $.  In the notation of Lemma  \ref{L1} let $\X = \P$ and $u=-w$. Since $v \in  T^w_\P(x)\cap \P$ the lemma says $w \in N_\P(v)$. Then Lemma \ref{conebound} says $$\ds   \frac{a \cdot w }{\|a\|\|w\|} \ge  \min \left \{ \frac{a \cdot u}{ \|a\|\|u\|} : u \in N_\P(v)\right \}  = \theta_v.$$
	To reach a contradiction it is enough  to show $  \frac{a \cdot w}{\|a\|\|w\|} < \theta_v.$ To that end write
	\begin{align}  
	\frac{a \cdot w}{\|a\|\|w\|} =  \frac{1}{2}\left \| \frac{a}{\|a\|} + \frac{w }{\|w \|}\right \|^2-1 \label{a1}
	\end{align} 
	To see the right-hand-side is less than $\theta_v$ we claim $\left \| \frac{a}{\|a\|} \!+ \!\frac{w }{\|w \|}\right \| \!<\! \sqrt {2 \phi_v}$. To that end write $$ w= \left( y_{1} - \frac{\eta}{ \sqrt n} \Sum a_i \right)  - x  = y_{1} -  \eta\sqrt n a   + \eta \Ep - x  =  \left( y_{1}    + \eta \Ep  - x \right)  -  \eta\sqrt n a  = \dd -\eta \sqrt n a$$
	for $\dd = y_{1}    + \eta \Ep  - x$. Hence for  $\dd ' = \dd / \eta \sqrt n$ we have $\frac{w}{\|w\|}  = \frac{ \dd -\eta \sqrt n a}{\| \dd -\eta \sqrt n a\|} =   \frac{ \dd' -  a}{\| \dd' -  a\|}$.	To simplify     the right-hand-side of \eqref{a1}  write
	\begin{align*} \frac{a}{\|a\|} + \frac{w}{\|w\|} = \frac{a}{\|a\|} + \frac{ \dd' -  a}{\| \dd' -  a\|} = \frac{a}{\|a\|} \left( 1 - \frac{\|a\|}{\|\dd' - a\|}\right)  + \frac{ \dd'  }{\| \dd' -  a\|} \\
	= \frac{a}{\|a\|} \left( \frac{\|\dd' - a\|}{\|\dd' - a\|} - \frac{\|a\|}{\|\dd' - a\|}\right)  + \frac{ \dd'  }{\| \dd' -  a\|} 
	\end{align*}
	and use the reverse triangle inequality to see
	\begin{align} \left \| \frac{a}{\|a\|} + \frac{w}{\|w\|} \right \| \le    \left | \frac{\|\dd' - a\|}{\|\dd' - a\|} - \frac{\|a\|}{\|\dd' - a\|}\right|   + \frac{ \|\dd'\|  }{\| \dd' -  a\|} \le  2 \frac{  \|\dd'\|  }{\| \dd' -  a\|} =  2 \frac{  \|\dd\|  }{\| \dd  -  \eta \sqrt n a\|}\label{c0.5}
	\end{align} 
	To bound the numerator of \eqref{c0.5} use the $\phi_v$ terms in \eqref{smallnoise} to bound $\|\P\|$ and $\|\Ep\|$ and see
	
	\begin{align} 
	 \|\dd\| \le  \|y_1-x_n\|+\eta \|\Ep \|  \le \|\P\| +  \eta \|\Ep \| 
	   < \frac{\eta \sqrt n \|a\|}{\alpha} \sqrt{2 \phi_v} + \eta \beta \sqrt n \|a\|\sqrt{2 \phi_v}\notag \\  =  \left(\frac{1}{\alpha}+ \beta \right)\eta  \sqrt n \|a\|\sqrt{2 \phi_v} = \frac{1}{3}\eta  \sqrt n \|a\|\sqrt{2 \phi_v}\label{c0.75} .
	 \end{align}
 To bound the denominator of \eqref{c0.5} use  the $1$ terms in \eqref{smallnoise} to see $\|\dd\| < \frac{1}{3}\eta  \sqrt n \|a\|$ and so
	  
	 \begin{align} 
	 \| \dd  -  \eta \sqrt n a\| > \eta \sqrt n \|a\| - \|\dd\|  > \eta \sqrt n \|a\| - \frac{ \eta \sqrt n \|a\|}{3} = \frac{ 2\eta \sqrt n \|a\|}{3}. 
	 %\eta \sqrt n \|a\| - \|X\| >  \eta \sqrt n \|a\| -\frac{ \eta \sqrt n \|a\|}{3} = \frac{2\eta \sqrt n \|a\|}{3} 
	 \label{c2}
	 \end{align}
	 Combine (\ref{c0.5}$-$\ref{c2}) to conclude $\left \| \frac{a}{\|a\|} \!+ \!\frac{w }{\|w \|}\right \| \!<\! \sqrt {2 \phi_v}$.
	\end{proof}
Next we relate the previous lemma to the regret.
  
  %$o \in \{1,2,\ldots, V\}$. Write $\DD = \min \{\DD_j: \DD_j >0\}$. It follows $\DD = \DD_1  = \ldots =  \DD_o < \DD_{o+1} \le \ldots \le \DD_V$.
  
\begin{lemma}\label{ass}\normalfont Suppose the conditions   \eqref{smallnoise} hold for some $n \in \NN$ and $v \in \ol \V_j$. Then for each $v \in  \V^*$ we have  $a \cdot(x_{n+1} - v^*) \le \DD_{j-1}$.
\end{lemma}

\begin{proof}   Lemma \ref{assumptions} says $v \notin T^{w}_\P(x_{n+1}) $ for  $w  = y_{n+1}-x_{n+1}$. Moreover since  $\DD_2 < \ldots < \DD_d$ and $\phi_2 < \ldots < \phi_d$ the conditions \eqref{smallnoise} also hold for each $v \in \ol \V_k$ with $k \in \{j, \ldots, U\}$. Hence  $ T^{w}_\P(x_{n+1}) \cap \V \subset \ol \V_j$.  Now consider the face $F = T^{w}_\P(x_{n+1})  \cap \P$ of the domain. Lemma \ref{face} says $F$ is the convex hull of some subset of $\V$.  Since $ T^{w}_\P(x_{n+1}) \cap \V \subset \ol \V_j$ we see $F$ is contained in the convex hull $C$ of $ \ol \V_j$.  By linearity we have $a \cdot(x - v^*) \le \DD_{j-1}$ for each $x \in C$. Since $x_{n+1} \in F \subset C$ we get $a \cdot(x_{n+1} - v^*) \le \DD_{j-1}$ as required. 
%Since $F$ contains no $v_j,v_{j+1},\ldots, v_d$ and no vertex is contained in the convex hull of the others $F$ must be contained in the convex hull of $\{v^*,\ldots, v_{j-1}\}$. In particular $x_{n+1}$ is contained in the hull of $\{v^*,\ldots, v_{j-1}\}$. 
%$\{v^*,v_2,\ldots, v_{j-1}\}$. Since $x_{n+1} \in Q$ we have $a \cdot(x_{n+1} - v^*) \le \DD_{j-1}$.
\end{proof}
We will use the following vector-concentration inequality of \cite{GoodAH} to show the small noise condition  \eqref{smallnoise} holds with high probability as $n \to \infty$.  See Appendix C for discussion of the theorem.

\begin{theorem} \normalfont\label{Pinelis}Suppose the i.i.d sequence $X_1,X_2,\ldots$ takes values in $\RR^d$. Suppose each $\Ex[X_i] = 0$ and $\|X_i\| \le R$. Then for each $r \ge 0$ we have
	$$P \left( \Big \|\sum_{i=1}^n X_i \Big\| \ge  n r \right) \le 2\exp \left( -\frac{r^2}{2 R^2} n\right).$$
\end{theorem} 
Similar to \eqref{smallnoise} let $\AA > 3$   and $  \BB =   \frac{1}{3} - \frac{1}{\AA}$. For each $j \in \{2,\ldots, U\}$ we define \begin{align}n_j =  \left \lceil \left ( \frac{\AA \|\P\| }{\eta \|a\| }\right )^2  \left ( 1+ \frac{D^2 \|a\|^2}{\DD_j^2}\right) \right \rceil +1 \qquad  \qquad r_j =  \BB   \|a\| \left(  1+ \frac{D^2 \|a\|^2}{\DD_j^2}\right)^{\hspace{-1mm}-1/2} \label{r}\end{align} 

The expression for $n_j$ mirrors  the bound for $n$ in  \eqref{smallnoise}. The $\min$ is replaced with a sum and $\phi_v$ is written explicitly following \eqref{angles}. The expression for $r_j$ mirrors the bound for $\Ep_{n+1}$ without the factor of $\sqrt n$.
We will use the above to  derive separate bounds for the   initial segment   $ \sum_{i=1}^{n_2} a \cdot(x_i -v^*)$ and the final segment $ \sum_{i> n_2}^{\infty} a \cdot(x_i -v^*)$ of the pseudo-regret.  First we bound the probabilities for  the final segment.

%Like before we derive separate bounds over an initial and final segment  $\{1,2,\ldots, N\}$ and $\{N+1.N+2,\ldots \}$ of the turns. For the final segment Theorem \ref{Pinelis} combined with Lemma \ref{ass} gives the following bound.
\begin{lemma}\label{concentration} \normalfont
  Let  $j \in \{2,\ldots, U\}$ and $\AA,\BB$ and $n_j, r_j$ be as defined in {\normalfont (\ref{r})}. For  $n>n_j$  we have
 \begin{align*}
  P\big ( a \cdot (x_{n+1} - v^*) > \DD_{j-1}\big ) \le 2 \exp \left (-\frac{ r_j^2}{2R^2} n\right). 
 \end{align*}

 \noindent In particular for $j=2$ we have \vspace{-2mm}
 \begin{align*}
  P\big ( a \cdot (x_{n+1} - v^*) >0 \big ) \le  2 \exp \left (-\frac{r_2^2}{2R^2} n\right). 
 \end{align*}
\end{lemma}

\begin{proof}Lemma \ref{conebound} and the definitions \eqref{r} give
\begin{align} n_j > \left ( \frac{ \AA \|\P\| }{\eta \|a\| }\right )^2 \max \left \{1, \frac{1  }{  2\phi_v  }\right\} \qquad \qquad \sqrt n r_j \le  \BB \sqrt n \|a\| \min \left \{1  ,  \sqrt {2\phi_v} \right \}.\label{smallnoise1}\end{align}
Note the right-hand-sides of \eqref{smallnoise1} are the same as \eqref{smallnoise}. 
Hence Lemma \ref{ass} says it is enough to show $P \left(  \|\Ep_{n+1}  \| \ge   {\sqrt n} r_j \right) \le 2\exp \left( -\frac{r^2_j}{2 R^2} n\right).$  To that end use Theorem  \ref{Pinelis}  with $X_i = a_i-a$ to see $P \left( \big \|\sum_{i=1}^{n} (a_i-a) \big\| \ge n r_j \right) \le 2\exp \left( -\frac{r^2_j}{2 R^2} n\right).$  To finish recall  $ \Ep_{n+1} = \frac{1}{\sqrt{n}}\sum_{i=1}^{n} ( a-a_i)$ is the error term so the left-hand-side equals $P \left( \|\Ep_{n+1}\| \ge \sqrt{n} r_j \right)$. %The definition of $n_j$ and bounds for $\phi_j$ in Lemma \ref{conebound} ensures the first condition of \eqref{smallnoise} is met.   Hence by Lemma \ref{ass} it is enough to show $\sqrt n r_j \ge \BB \sqrt {n_j} \|a\| \min\{1, \sqrt{2 \phi_j}\}$. This follows from the definition of $r_j$ and Lemma \ref{conebound}.
\end{proof}
Next we use the above to bound the expectation of the final segment.
\begin{lemma}\normalfont \label{finalP}Let  $j \in \{2,\ldots, U\}$ and $\AA,\BB$ and $r_j,n_j$  be as defined in {\normalfont (\ref{r})}. We have
\begin{align*} \sum_{n=n_2}^\infty \Ex\big[a \cdot (x_{n+1}-v^*) \big]  \le  \frac{12 R^2D^2}{\BB^2 \DD }  \exp \left (-\frac{1}{2}\left(\frac{\AA \BB \|\P\| }{\eta  R } \right)^2 \right).\end{align*} 
\end{lemma}
\begin{proof}Lemma \ref{concentration} says the complementary CDF $F(t) = P(a \cdot (x_{n+1}-v^*) >t)$ is dominated by the piecewise function 
	$$f(x) =  \begin{cases} 
\ds  2 \exp \left (-\frac{r_2^2}{2R^2} n\right)  & 0 <  x \le \DD_{2}\\
\ds 2  \exp \left (-\frac{ r_{k}^2}{2R^2} n \right)  & \DD_{k-1} < x \le \DD_k  \text{ with } k \ge 3\\[5pt]
\ds 0 &   \DD_U < x
\end{cases}
$$ 
The second part of  Lemma  \ref{CDF} says $\Ex\big[a \cdot (x_{n+1}-v^*) \big] = \int_0^\infty F(t) dt$ and so
			\begin{align*} 
			\Ex\big[a \cdot (x_{n+1}-v^*) \big] &\le \int_0^\infty f(t)dt= \int_0^{\DD_U} f(t)dt\\ &= 2\DD_{2} \exp \left (-\frac{ r_2^2}{2R^2} n\right)     + 2\sum_{k=3}^{U} (\DD_{k}-\DD_{k-1})\exp \left (-\frac{ r_k^2}{2R^2} n\right)  .
			\end{align*}		
			Now take the sum from  $n_2$ to $\infty$. The terms are  decreasing so we can bound the sums by the corresponding integrals and get
\begin{align*}\sum_{n=n_2}^\infty \!\Ex \big[a \cdot (x_{n+1}\!-\!v^*) \big] & \le 2\DD_{2} \!\sum_{n=n_2}^\infty \! \exp \left (-\frac{ r_2^2}{2R^2} n\right)     + 2\sum_{n=n_2}^\infty\sum_{k=3}^{U} (\DD_{k}-\DD_{k-1})\exp \left (-\frac{ r_k^2}{2R^2} n\right) \\ & \le 2\DD_{2} \! \int_{n_2-1}^\infty \! \exp \left (-\frac{ r_2^2}{2R^2} x\right) dx     + 2 \!\int_{n_2-1}^\infty \! \sum_{k=3}^{U} (\DD_{k}-\DD_{k-1})\exp \left (-\frac{ r_k^2}{2R^2} x\right)dx\\ &=   4R^2\frac{\DD_{2}}{r^2_2}  \exp \left (-\frac{ r_2^2}{2R^2} (n_2-1)\!\right)      + 4R^2 \! \sum_{k=3}^{U} \! \frac{\DD_{k}-\DD_{k-1}}{r^k_2}  \exp \! \left (-\frac{ r_k^2}{2R^2} (n_2-1) \! \right)   \end{align*} 
Since $\DD_2 \le \ldots \le \DD_U$ we have all $r_2 \le r_k $ and the above gives

\begin{align} \sum_{n=n_2}^\infty \Ex\big[a \cdot (x_{n+1}-v^*) \big] &\le  4R^2\left( \frac{ \DD_{2}}{r^2_2}+  \sum_{k=3}^U \frac{  \DD_{k}-\DD_{k-1} }{r^2_k}\right) \exp \left (-\frac{ r_2^2}{2R^2} (n_2-1)\right)\notag \\
			& \le 4R^2\left( \frac{ \DD_{2}}{r^2_2}+  \sum_{k=3}^U \frac{  \DD_{k}-\DD_{k-1} }{r^2_k}\right) \exp \left (-\frac{ r_2^2}{2R^2} \left ( \frac{\AA \|\P\| }{\eta \|a\| }\right )^2  \left ( 1+ \frac{D^2 \|a\|^2}{\DD ^2}\right)\right)\notag \\
			& = 4R^2\left( \frac{ \DD_{2}}{r^2_2}+  \sum_{k=3}^U \frac{  \DD_{k}-\DD_{k-1} }{r^2_k}\right) \exp \left (-\frac{1}{2}\left(\frac{\AA \BB \|\P\| }{\eta  R } \right)^2 \right)\label{ex1}
			\end{align} 
			
where the last line follows from expanding the definition (\ref{r}) of $r_2$ and cancelling terms. To bound the second factor  in \eqref{ex1} expand the definition of each  $r_k$ to get

 \begin{table}[!h] \addtolength{\tabcolsep}{-4pt}
 	\centering
 	\begin{tabular}{rrrr} 
 	$\ds \frac{  \DD_{k}-\DD_{k-1} }{r^2_k}  =$&$\ds (\DD_{k}-\DD_{k-1}) \frac{   1+ D^2\|a\|^2/\DD_k^2 }{ \beta^2 \|a\|^2}  = $ & $\ds \frac{   \DD_{k}-\DD_{k-1} }{ \beta^2 \|a\|^2}   $ &$\ds +\ (\DD_{k}-\DD_{k-1}) \frac{    D^2  }{ \beta^2 \DD_k^2}.$\\[20pt]
 	$\ds  \frac{  \DD_{2}  }{r^2_2} =$	    &$\ds  \DD_{2}  \frac{   1+ D^2\|a\|^2/\DD_2^2 }{ \beta^2 \|a\|^2}=   $&$\ds        \frac{   \DD_{2}  }{ \beta^2 \|a\|^2}   $&$\ds   +  \     \frac{    D^2  }{ \beta^2 \DD_2}$\\[20pt]
 	\end{tabular} 
 \end{table}    \vspace{-5mm}			  
\noindent The sum over the first terms telescopes to give $\frac{\DD_U}{\BB^2 \|a\|^2}$. The sum over the second terms gives $\frac{D^2}{\beta^2} \left ( \frac{1}{\DD_2} +\frac{\DD_3-\DD_2}{\DD_3^2} + \ldots + \frac{\DD_U-\DD_{U-1}}{\DD_U^2} \right) $.
Lemma \ref{telescope} in Appendix D says the second factor is at most  $\frac{2}{\DD_2} = \frac{2}{\DD } $. Hence   \eqref{ex1} gives
\begin{align*}  \sum_{n=n_2}^\infty \Ex\big[a \cdot (x_{n+1}-v^*) \big]  \le   4R^2 \left( \frac{\DD_U}{\BB^2 \|a\|^2} +    \frac{2D^2 }{ \BB^2\DD}  \right)  \exp \left (-\frac{1}{2}\left(\frac{\AA \BB \|\P\| }{\eta  R } \right)^2 \right).\end{align*}
To remove the $\DD_U$ term   recall $\DD_U   = a \cdot(v - v^*)$ for some $v \in \V$. Hence $ \DD_U \le \|a\|D$ and   $\frac{1}{\|a\|^2} \le \frac{D^2}{\DD_U^2} $. Hence we have  $\frac{\DD_U}{\BB^2 \|a\|^2} \le \frac{D^2}{\BB^2 \DD_U}  \le \frac{D^2}{\BB^2 \DD }$. Gather common factors to complete the proof. \end{proof}
%\begin{align*}   \frac{144 R^2 \big ( 1+ D^2\|a\|^2/\DD^2 \big)}{ \|a\|^2} =  \frac{144 R^2  }{ \|a\|^2} +  \frac{144 R^2 \big ( 1+ D^2\|a\|^2/\DD^2 \big)}{ \|a\|^2} = 144 R^2 \left ( \frac{1}{\|a\|^2} + \frac{D^2}{\DD^2}\right ) 
%\end{align*}
Next we bound the expectation of the initial segment.

\begin{lemma}\label{initialP} \normalfont Let  $j \in \{2,\ldots, U\}$ and $\AA,\BB$ and $r_j,n_j$  be as defined in  {\normalfont (\ref{r})}. We have 

\begin{align*}
\Ex \left [\sum_{i=1}^{n_2} a  \cdot (x_i - v^*) \right] \le     LD + \left ( \frac{\|\P\|^2}{2 \eta}  + 2 \eta L^2 +  \sqrt{ \frac{\pi}{2} }RD \right ) \left(  \frac{3}{2}\frac{ \AA D^2 }{ \eta   }      + \frac{\eta  L^2 }{\AA   }   \right )\frac{1}{\DD}. \end{align*}
\end{lemma}

\begin{proof}Theorem 1 says \begin{align*}
\sum_{i=1}^{n_2} a_i \cdot (x_i - v^*) \le     LD + \left ( \frac{\|\P\|^2}{2 \eta}  + 2 \eta L^2 \right) \sqrt {n_2} .\end{align*}
By Lemma \ref{MartingaleBound} in Appendix B  we have 
\begin{align*}\Ex \left [\sum_{i=1}^{{n_2}}  (a-a_i) \cdot( x_i-v^*)  \right ] & \le   \sqrt{ \frac{\pi}{2} }RD \sqrt {n_2}.  
\end{align*}Take expectation and add the two lines together to get
\begin{align} \label{thing}
\Ex \left [\sum_{i=1}^{n_2} a  \cdot (x_i - v^*) \right] \le     LD + \left ( \frac{\|\P\|^2}{2 \eta}  + 2 \eta L^2 +  \sqrt{ \frac{\pi}{2} }RD \right ) \sqrt {n_2}.  \end{align}
It remains to bound $\sqrt {n_2}$. The definition (\ref{r}) says
\begin{align*}\sqrt {n_2} \le \sqrt {\left ( \frac{\AA D }{\eta \|a\| }\right )^2  \left ( 1+ \frac{D^2 \|a\|^2}{\DD ^2}\right)+2}. 
\end{align*}
By concavity we have $\sqrt{x+a} \le \sqrt x + \frac{a}{2 \sqrt x}$ and the above is at most 
\begin{align*}  \frac{\AA D }{\eta \|a\| }  \sqrt {   1\!+ \!\frac{D^2 \|a\|^2}{\DD ^2} } +  \frac{\eta \|a\| }{\AA D } \left(   1\!+\! \frac{D^2 \|a\|^2}{\DD ^2} \right)^{\hspace{-1mm}-1/2 }
	 \! \le    \frac{\AA D }{\eta \|a\| }  \sqrt {   1+ \frac{D^2 \|a\|^2}{\DD ^2} } +  \frac{\eta \|a\| }{\AA D } \left(     \frac{D^2 \|a\|^2}{\DD ^2} \right)^{\hspace{-1mm}-1/2}\\
	 =  \frac{\AA D }{\eta \|a\| }  \sqrt {    \frac{D^2 \|a\|^2}{\DD ^2} +1} +  \frac{\eta  \DD }{\AA D^2 } \le     \frac{\AA D }{\eta \|a\| }  \left( \sqrt {    \frac{D^2 \|a\|^2}{\DD ^2}} + \frac{1}{2}\sqrt {    \frac{\DD ^2}{D^2 \|a\|^2}}\right)  +  \frac{\eta  \DD }{\AA D^2 }\\ =  \frac{\AA D^2 }{\eta \DD }  + \frac{\AA \DD}{2 \eta \|a\|^2 }   + \frac{\eta  \DD }{\AA D^2 }.  
	\end{align*}
	%Again by concavity of the  square-root we have
	%\begin{align*} \sqrt {   1+ \frac{D^2 \|a\|^2}{\DD ^2} }\le \frac{D  \|a\| }{\DD } + \frac{\DD }{2D  \|a\| }.
	%\end{align*}
	%The second has 
	%\begin{align*}
	 % \left(   1+ \frac{D^2 \|a\|^2}{\DD ^2} \right)^{-1/2} \le  \left(     \frac{D^2 \|a\|^2}{\DD ^2} \right)^{-1/2} = \frac{\DD }{ D  \|a\| }
	%\end{align*}
%Hence we have
%	\begin{align*}\sqrt {n_2} &\le   \frac{\AA D^2 }{\eta \DD }  + \frac{\AA \DD}{2 \eta \|a\|^2 }   + \frac{\eta  \DD }{\AA D^2 }.  
%	\end{align*}	
To put the three terms in the same form recall  $\DD     =  a  \cdot\!(v_2 \!- \! v^*) \! \le \!\|a\| D$ and so $\frac{1}{\|a\|} \!\le \!\frac{D}{\DD} $ and $\frac{1}{D} \! \le \!\frac{\|a\|}{\DD} \! \le \! \frac{L}{\DD}$. Hence the second term is at most $ \frac{\AA \DD}{2\eta } \frac{D^2}{\DD^2} \! = \! \frac{  \AA D^2}{2 \eta \DD}$ and the third term is at most $\frac{\eta \DD }{\AA  }\frac{L^2}{\DD^2}  \! = \! \frac{\eta L^2}{\AA \Delta}$.  Hence $n_2 \! \le \! \frac{3}{2}\frac{ \AA D^2 }{ \eta \DD }    \!  + \! \frac{\eta  L^2 }{\AA \DD } $. Plug this into   (\ref{thing}) and simplify to complete the proof.
\end{proof}
The main theorem now follows from combining the bounds in Lemma \ref{initialP} and \ref{finalP} for the initial and final segments.  
\vspace{5mm}

\begin{proof}{\bf of Theorem \ref{T2}}
 The first and second lines of \eqref{firstbound} come from Lemmas  \ref{initialP} and \ref{finalP} respectively. To prove \eqref{secondbound} plug $\eta = D/2L$ into \eqref{firstbound}  and use  $\|\P\|\le D$ to get

 \begin{align} \Ex \!\left [\sum_{i=1}^{\infty} \! a \!\cdot\!( x_i \!-\!v^*)   \right ]   \le       L D + \left (  2L   + \! \sqrt{ \frac{\pi}{2} }R  \! \right ) \left( \!  3 \AA + \frac{1}{2\AA} \!\right) \frac{LD^2}{\DD}   \!+ \!  \frac{12R^2D^2}{\BB^2\DD}  \! \exp \left (\! -2\left(\frac{\AA \BB L}{   R^2 }  \right)^2\right). \label{proof1}\end{align}
 
 %Now we remove dependence on $\|\P\|$.  Since the ball  at $y_1 \in \P$ with radius $\|\P\|$ contains $\P$ we have $D/2 \le \|P\|$. Moreover at least one vertex lies on the boundary of the ball, as otherwise we could shrink the radius slightly which contradicts the definition of $\|\P\|$. Hence we have $\|\P\| \le D$. 

 %Since $\DD = a \cdot(v_2-v^*) \le LD$ we can bound the term $\ds \frac{   \DD }{2\AA D  L }    \le \frac{1}{2\AA}$. Since we can take $R = 2L$ we can replace the exponent $2\left(\frac{\AA \BB L}{   R^2 }  \right)^2$ with $\frac{\AA^2 \BB^2}{2} $.
 % follows (\ref{combine}) is at most 

 %\begin{align*} L D+
 %\left (  2L D  +  \sqrt{ \frac{\pi}{2} }RD \right ) \left(    \frac{3\AA D  }{   \DD }  L    + \frac{1}{2\AA }   \right ) +   \frac{12R^2D^2}{\BB^2\DD}   \exp \left (-\frac{ \AA^2 \BB^2 }{  2 }  \right).
 %\end{align*} 
%The bound is difficult to optimise algebraically. 
 %We can make the further simplification  $R\le 2L$ to get the bound 
 %\begin{align*}
    %\left ( 2LD +  \sqrt{ \frac{\pi}{2} }RD \right ) \left(  \frac{3\AA LD }{  \DD }      + \frac{  \DD }{2 \AA L D  }   \right ) +   \frac{12R^2D^2}{\BB^2 \DD} \exp \left (-\frac{1}{2} \AA^2 \BB^2 \right)
 %\end{align*}

 Since we can take $R=2L$ we can replace the exponential with $\exp \left (-\frac{(\AA \BB)^2}{2}\right)$. For $\AA = 10$ we have  $\BB = \frac{1}{3}-\frac{1}{10} = \frac{7}{30} $ and $\AA \BB =  \frac{7}{3}$. The coefficients are bounded by $\left(3 \AA + \frac{1}{2\AA} \right) = 30 + \frac{1}{20} \le 31$ and $\frac{12}{\BB^2} = 12\left( \frac{30}{7}\right)^2  = \frac{10800}{49} \le 221 $ and $\exp \left (-\frac{(\AA \BB)^2}{2}\right) = \exp \left (-\frac{49}{18}\right) \le \frac{15}{221}$.
 Plug these bounds into  \eqref{proof1} and simplify to prove \eqref{secondbound}.
 
 %\begin{align*}   \left(3 \AA + \frac{1}{2\AA} \right) = 30 + \frac{1}{20} \le 31 && \frac{12}{\BB^2} = 12\left( \frac{30}{7}\right)^2  = \frac{10800}{49} \le 221 \end{align*}
 %\begin{align*}   \exp \left (-\frac{(\AA \BB)^2}{2}\right) = \exp \left (-\frac{49}{18}\right) \le \frac{15}{221} \end{align*}
 %Plug the above into   \eqref{proof1} to prove \eqref{secondbound}.
\end{proof}
Note the bound \eqref{firstbound} holds simultaneously for all hyperparameters $\AA>3$ but is difficult to optimise algebraically. In \eqref{secondbound}  we choose $\AA=10$ to give the coefficients the same order.  
%In the above we select $y_1 \in \P$ to maximise $\|\P\| = \max\{\|y_1-x\|: x \in \P\}$. The more obvious choice is to instead minimise $\|\P\|$ as this puts the basepoint at the {\it centre} of $\P$ rather than its furthest extremity. 
%However it is not obvious from Theorem \ref{T2} that this improves performance. Taking $\ds y_1$  with  $\ds \|\P\| =   \min_{y \in \P} \max_{x \in \P} \|y-x\|$ and $\eta = \|\P\|/2L$ leaves the third term of (\ref{T2bound}) unchanged  but creates factors of $\|\P\|$ in both the numerators and denominators of the second term. It is not obvious this improves the bound.

%Since Theorem \ref{T2} holds for $R=2L$ we have the order bound.

%\begin{corollary}
% Under the hypotheses of the second part of Theorem \ref{T2} running Algorithm 1 with base point $y_1 \in \P$ and parameter $\eta = D/2L$ gives pseudo-regret of order $O(L^2 D^2/\DD)$ independent of the dimension.
%\end{corollary}
  
 \section{Examples  with Euclidean Bounds}

 \begin{table}[!b]\def\arraystretch{1.4}
 	\centering
 	\begin{tabular}{l|c|c|c| r r r }
 		 
\multicolumn{1}{c}{ }& 
\multicolumn{3}{c}{Dimensions}&
\multicolumn{1}{r}{Algorithm}&
\multicolumn{1}{c}{Antagonistic regret}&
\multicolumn{1}{c}{i.i.d pseudo-regret }  \\[5pt] \hline 
 
 Poltope               & $D^2$ & $V$  & $L_\infty$    & Gradient Descent & $ L D \sqrt N  $ & $ L^2D^2/\DD $  \\ 
 		&  &  &    &  Hedge & $ L_\infty \sqrt{\log (V) N} $  & $ L_\infty ^2 \log (V) /\DD  $   \\  \hline
 		$d$-Simplex           & $ 2$ & $d$ & $L$ & $\ \ $ Gradient Descent & $ L \sqrt {N}  $ & $ L^2 /\DD   $  \\
 		&&& &Hedge &$ L \sqrt { \log (d)N}  $ & $ L^2 \log d /\DD  $ \\  [5pt]
 		$d$-Cube              & $4 d$ & $2^d$ &  $L \sqrt d$ & Gradient Descent  &$ L \sqrt {dN}  $ & $  L^2 d/ \DD  $   \\
 		&&&&Hedge & $ L d\sqrt {N}  $ & $ L^2 d^2/\DD  $ \\[5pt] 
 		$\B(n)$  & $2n$  & $n!$ &  $L \sqrt n$ & Gradient Descent & $ L  \sqrt {nN}   $& $  L^2 n/\DD  $ \\ 
 		&&&&Hedge &$ L n\sqrt { \log (n)N}   $ & $ L^2 n^2 \log n/\DD $ \\  [5pt]
 		$\P(d)$ &  $\ds \frac{d^3}{3}$ & $d!$ &  $\ds \frac{ L\, d^{\, 3/2}}{\sqrt 3}$ & Gradient Descent &$ L d^{3/2} \sqrt {N}  $ & $  L^2 d^3 /\DD  $     \\ 
 		&&&&Hedge &$ L \,d^{\,5/2}\sqrt {  \log(d)N}  $ & $  L^2 d^5 \log d/\DD  $ \\  [5pt]
 		$\P_\pm (d)$  &  $ \ds \frac{4d^3}{3}$ & $2^d d!$ & $\ds \frac{ L\, d^{\, 3/2}}{\sqrt 3}$  & Gradient Descent &$ L \, d^{3/2}\sqrt {N}   $ & $  L^2 d^3/\DD  $  \\
 		&&&&Hedge &$ L \,d^{\,5/2}\sqrt{  \log(d)N}   $ & $ L^2 d^5 \log d/\DD  $ \\  [5pt] \hline 
 	\end{tabular} 
 	
 	\captionsetup{format=hang} 
 	\caption{ Comparison of order bounds for Gradient Descent and Lifted Hedge under Euclidean bounds  $\|a_n\| \le L$ on cost vectors. Here $L_\infty$ bounds the $\infty$-norm of the lifted vectors.}
 	
 \end{table}    
 
 \noindent In this section we compare our results from  Theorem \ref{T2} for Gradient Descent to those of \cite{OptimalHedge} from lifting the problem and running Hedge.  
 
 We pay special attention to the Birkhoff Polytope and (signed) permutahedron, as these are are particularly well-studied in the context of optimisation. See for example  \cite{TransportPolytopes3,TransportPolytopes2,TransportPolytopes6,TransportPolytopes7,TransportPolytopes5,TransportPolytopes1,TransportPolytopes8} and the references therein.  The Birkhoff polytope $\B(n)$ is the convex hull of the $n \times n$ permutation matrices and the permutahedron $\P(d)$ is the convex hull of the vectors with components $1,2,\ldots, d$. For full definitions see Section 4 Example 1.
 These polytopes occur in problems where each turn we must select a permutation. For example to rank user preferences, choose a route through a graph, or assign  ``vehicles'' to ``routes'' in a transportation problem. For further examples  see  \cite{warmuth2008randomized}  and \cite{kalai2016efficient}. 
 For overviews see    \cite{CombinatorialOptimisation} or \cite{TransportationTextBook}.
  
 Before comparing performance, we describe the lifting procedure in detail. Given a polytope domain   $\P \subset \RR^d$ with vertices  $\{v_1,v_2,\ldots, v_V\}$ and cost vectors $a_1,a_2,\ldots$ we  define an auxiliary problem on the $V$-simplex. Let $\phi: \RR^V \to \RR^d$ be the unique linear map with each $\phi e_j=v_j$. Define the auxiliary cost vectors $A_1,A_2,\ldots \in \RR^V$ with components $A_i(j) = a_i \cdot v_j$. Running Hedge on the auxiliary problem gives actions  $X_1,X_2,\ldots $ in the $V$-simplex.  The results of \cite{OptimalHedge} say these actions give  $O\big(L_\infty \sqrt {\log (V) N }\big)$ regret in the antagonistic case and  $O(L_\infty^2 \log(V)/\DD )$ pseudo-regret in the i.i.d case. By linearity the actions $x_n = \phi X_n$ in the original problem satisfy the same regret bounds.  Since $\P$ is the convex hull of its vertices  $x_n$ are valid actions in the original problem.
 
 To bound $L_\infty = \sup \{\|A_i\|_\infty : i \le N\}$ in terms of the given quantity $L$ write $\|A_i\|_\infty = \max_i |A_i(j)| = \max _j |a_i \cdot v_j| \le \max_j \|a_i\|\|v_j\|  $ and use Cauchy-Schwarz to get  $L_\infty \le  L\|\P\|$.  This is used to express the Hedge bounds in Table 1 in terms of $L$ rather than $L_\infty$. 
 
 The Gradient Descent bounds in Table 1 come from Theorems \ref{worstcase} and \ref{T2}. Note the theorems are dimension-independent and contain only $L,D,N$ and not $d$. However the later  polytopes have $D$ grow with $d$ and hence the final bounds grow with dimension. 

 For the first three polytopes the $D^2$ and $L_\infty$ values in the table are exact. For $\P(d)$ and $\P_\pm(d)$ the $D^2$ values are limits as $d \to \infty$. For the precise values see   Examples 4.5 and 5.5 in Appendix A. The $L_\infty$ values are also limits of the exact value  $L_\infty =  L \sqrt{\frac {d(d+1)(2d+1)}{6}} $ obtained (\cite{Sums}) from  the formula $\sum_{n=1}^d n^2 =\frac {d(d+1)(2d+1)}{6} $.

%In the first column $D$ is the diameter of the polytope. For the first three examples $D$ is exact while for $\P(d)$ and $\P_\pm(d)$ it is exact as $d \to \infty$. In the second column $V$ is the number of vertices.  In the first three examples  $L_\infty$ is exact. For $\P(d)$ and $\P_\pm(d)$ the exact value is $L_\infty =  L \sqrt{\frac {d(d+1)(2d+1)}{6}} $ which is obtained \citep{Sums} using the formula $\sum_{n=1}^d n^2 =\frac {d(d+1)(2d+1)}{6} $. The bounds in the last two columns are order-bounds. 
   
   \subsection{Discussion}
   
Gradient Descent has a better dependence on dimension than Hedge for all the polytopes in Table 1. This is because $d$ affects the Hedge bounds twice. First through the explicit $\log(V)$ factor and second through $L_\infty$ which is dimension dependent in the later examples. For Gradient Descent the dimension only contributes once as the diameter grows with dimension.

\section{ Intrinsic Bounds on the Cost Vectors}
 
The bounds in  Theorems \ref{worstcase} and \ref{T2}   fall short in the special case when the cost vectors are nonzero but are perpendicular to the affine hull of the domain.  For example the simplex is contained in the subspace ${\{x \in \RR^d: x \cdot \1 = 1\}}$. Hence any minimisation problem on the simplex is trivial if all the cost vectors are multiples of $\1 = (1,1,\ldots, 1)$. However this triviality is not reflected in the theorems if the cost vectors are nonzero.

This can be remedied if we observe for Gradient Descent that all  behaviour of interest takes place inside the affine hull. In particular if we replace the cost vectors with their projections onto the affine hull, the actions and regret are unchanged. In the above example this gives a problem with all zero cost vectors. 

%In general the length of the projected cost vectors is reduced but nonzero.  
To that end we introduce the following intrinsic bounds.  Unlike Euclidean bounds the following do not depend on the choice of embedding $\X \subset \RR^d$. 
\begin{align} \label{nat} \sup \big \{|a_n \cdot(x-y)|:x,y \in \X \big \} \le L_\infty \qquad \sup \big \{|(a_n-a) \cdot (x-y)|:x,y \in \X \big \} \le R_\infty  
\end{align}  
 
The above generalises the  standard $\infty$-norm bounds on cost vectors for Hedge. For $\X$ the simplex the first bound is equivalent to each $|a_n(k)-a_n(j)| \le L_\infty$. Since the Hedge actions are unchanged by translating all components equally, we can replace each $a_n $ with $a_n - \frac{1}{2} \big(  \max_j   a_n(j)-\min_j a_n(j) \big) \1 $ to get $\|a_n\|_\infty \le L_\infty /2$ and likewise  $\|a_n -a\|_\infty \le R_\infty/2$.

In this section we will obtain regret bounds in terms of $L_\infty$ and $R_\infty$. The strategy is to first show the  Gradient Descent actions and regret are unchanged if we replace each $a_i$ with  $ \Pi_U(a_i) $ for $U$ the direction $\{t(x-y):x,y \in \X,t \in \RR \}$ of the domain. Next we convert (\ref{nat})  into Euclidean bounds for $ \|\Pi_U(a_i)\|$. Then we use Theorem \ref{T2} with the modified cost vectors to obtain an intrinsic bound for regret. %First we give  a geometric interpretation of the bounds (\ref{nat}).

%Suppose $\P = [-1,1]^2$ is the unit square. The quantity
%$\max \{|a_n \cdot (x-y)|: x,y \in \P\} = 2\big (| a_n(1)| + |a_n(2)|\,\big )$ is achieved by the vertices $x = ($sign$\,a_n(1), $ sign$\,a_n(2) )$ and $y=-x$. Hence the first bound in (\ref{nat}) becomes the $1$-norm bound  $\|a_n\|_1 \le L_\infty /2$. 
%More generally if   $\P$  is the rectangle $[-\AA,\AA] \times [-\BB,\BB]$ the set of allowed cost vectors is the $1$-ball rescaled to height $L_\infty /\BB$ and width $L_\infty /\AA$. If $\P$ is wide in some direction we allow only cost vectors with small components in that direction. Conversely  if $\P$ is narrow in some direction we allow cost vectors with large components in that direction.  
%This suggests the definition of the direction in which $\P$ is narrowest.

To convert  (\ref{nat}) into Euclidean bounds we consider the width of the domain. The width of a set $X \subset \RR^d$ with interior is the smallest distance $W$ such that $X$ can be sandwiched between two parallel hyperplanes distance $W$ apart. If the polytope has no interior we must first restrict attention to the affine hull and then consider hyperplanes.

\begin{definition}\label{widthdef} \normalfont  Let $\X \subset \RR^d$ be convex    with direction $U = \{t(x-y):x,y \in \X,t \in \RR \}$. For each $\ell \in \RR^d$ let $W_\ell$ be the length of the interval $\{ \ell \cdot x: x \in \X\}$. The width of $\X$ is defined as $W = \inf \{ W_\ell: \ell \in U,  \|\ell\|=1 \}$. 
\end{definition}

To the authors' knowledge the notion of width does not appear in the existing optimisation literature to describe the shape of an action set. It appears elsewhere, for example in the study of mean widths of simplices (see \cite{meanwidth} and the references within); discrete geometry (\cite{latticewidth}); and variants of Tarski's plank problem about covering a given convex set with copies of some prescribed shapes (see \cite{plankproblem} and the  references within).  Below are examples of widths of familiar polytopes. See Appendix A for proofs.\\ 
  
 \begin{example} \normalfont $ $ 
 %\noindent{\bf Examples}
 \begin{enumerate}
  \item[(1)]  The width of an $m$-dimensional cuboid $\prod _{j=1}^m [\AA_j,\BB_j]$ embedded in $\RR^d$ is  $\min_j|\BB_j-\AA_j|$.
  \item[(2)]  The $d$-simplex $ \{x \in \RR^d: \mbox{all } x(j) \ge 0  \mbox{ and } x(1) + \ldots + x(d) =1\}$ has width  $2 /\sqrt d$ for $d$ even. For $d$ odd the width is $2 /\sqrt d$  as $d \to \infty$. 
 \item[(3)]  The Birkhoff Polytope $\B(n)$ is the set of nonnegative $n\times n$ matrices with all row and column sums equal to $1$. Equivalently $\B(n)$  is the convex hull of the $n!$ permutation matrices. The width is bounded below by $ 2/\sqrt{n-1}$. 
 \item[(4)] The permutahedron $\P(d)$ is the set of vectors $x \in \RR^d$ with entries $\{x(1),\ldots, x(d)\} = \{1,2,\ldots, d\}$.  Equivalently $\P(d)$ is the convex hull of $\{(\si(1),\ldots, \si(d)): \si \in S_d \}$ for $S_d$ the permutation group. The width satisfies
 $$W \ge  \sqrt{\frac{5d^2 +8d +4 }{6}} \qquad \qquad   \liminf_{d \to \infty} \frac{W}{d} \ge \sqrt{5/6}.$$ 
 \item[(5)] The signed permutahedron $\P_\pm$ is the convex hull of the vectors $ (\pm\si(1),\ldots, \pm\si(d))$ for all choices of signs and permutation $ \si \in S_d $. The width satisfies
 $$W\ge  2\sqrt{  \frac{2d^2 +3d +1}{6}}  \qquad \qquad   \liminf_{d \to \infty} \frac{W}{d} \ge \sqrt{4/3}. $$ %We use the same  variance trick with $X=\ell \cdot \si$
 \end{enumerate}
 \end{example}
 As promised we start by relating the width to Euclidean bounds.

 \begin{lemma}\label{natbound}\normalfont  Let $\X \subset \RR^d$ be convex  with direction $U$ and width $W$. For each $c \in U$  we have $\|c\| \le \frac{1}{W}\sup\{|c \cdot(x-y)|:x, y  \in  \X\}.$ 
 \end{lemma}
 
 \begin{proof}
 By definition $W_\ell = \sup \big  \{|\ell \cdot(x-y)|:x,y \in \X \big \}$ for each unit vector $\ell$. For $\ell = c/\|c\|$ we get 
 \begin{align*}W_\ell     =   \sup\left \{\left|\frac{1}{\|c\|}\,  c\cdot(x-y)\right|:x,y \in \X \right \} 
 =   \frac{1}{\|c\|}\, \sup\big  \{\!\left|   c\cdot(x-y)\right|:x,y \in \X \big \} \\
 \implies \|c\| =  \frac{\sup \big \{ \!\left|c\cdot(x-y)\right|:x,y \in \X \big \} }{W_\ell} \le   \frac{\sup \big \{\!\left|   c\cdot(x-y)\right|:x,y \in \X \big\} }{W}
 \end{align*}where the inequality comes from how $W  \le W_\ell$ by definition of the width.  
 \end{proof}
 We wish to use Lemma \ref{natbound} to bound the length of the cost vectors.  Since we cannot assume the costs lie in the direction of the domain,  we must show (Lemma \ref{ort2}) the actions are unchanged if we replace each cost vector with its projection. The first step (Lemma \ref{ort}) is to show the projection onto a convex set factors through the projection onto its direction.

 \begin{lemma}\label{ort}\normalfont 
  Suppose the convex set $\X \subset \RR^d$ has direction $U$. For each $p \in \RR^d$ we have 
  $\Pi_\X(p) = \Pi_\X(\Pi_U(p))$.
 \end{lemma}

 \begin{proof} Recall $U+x$ is the affine hull of $\X$ for each $x \in \X$.  More generally suppose $U \subset \RR^d$ is a vector subspace with $\X \subset U + t$ for some $t \in \RR^d$.
 We claim $\Pi_\X(p) = \Pi_\X(\Pi_U(p))$ for each $p \in \RR^d$.
 To that end recall  $\Pi_\X(p) = \ds \arg \! \min_{x \in \X} \|p-x\|^2 $ and write $p-x = (p-y) + (y-x)$.  By definition the second term on the right is contained in $U$.
 Since $y = \Pi_U(p)$ the first term  is orthogonal to $U$. 
 Hence we have $\|p-x\|^2 = \|p-y\|^2 + \|y-x\|^2$ and  
 $\Pi_\X(p) \ds  = \arg \! \min_{x \in \X} \left( \|p-y\|^2 \!+\! \|y-x\|^2\right).$ Since the first term does not   depend on $x$ we have $\ds \Pi_\X(p) =  \arg \! \min_{x \in \X}   \|y-x\|^2  =  \arg \! \min_{x \in \X}   \|\Pi_U(p)-x\|^2   = \Pi_\X(\Pi_U(p)). $  \end{proof}  
 
\begin{lemma}\label{ort2} \normalfont Suppose the domain $\X$ has direction $U$. Let $c_1,c_2,\ldots $ be the projections of the cost vectors $b_1,b_2,\ldots$ onto $U$. The actions chosen by Algorithm $1$ given $c_1,c_2,\ldots $ are the same as those given $b_1,b_2,\ldots $.
 \end{lemma}

\begin{proof} %Let $y_0$ be the projection of the basepoint $x_0$ onto $\P$. 
	Given cost vectors $b_1,b_2,\ldots $ Algorithm 1 selects actions $$x_{n+1} = \Pi_\X \left(y_1 - \eta   \frac{    b_1 + \ldots + b_n}{\sqrt n} \right).$$ Lemma \ref{ort} says the right-hand-side is unchanged if we  replace the argument with its projection onto $U$. Since projection onto a vector subspace is a linear function we have 
	$$x_{n+1} =\Pi_\X \left(\Pi_U(y_1) - \eta   \frac{    \Pi_U(b_1) + \ldots + \Pi_U(b_n)}{\sqrt n} \right) = \Pi_\X \left(\Pi_U(y_1) - \eta   \frac{    c_1 + \ldots + c_n}{\sqrt n} \right).$$ 
	Since $c_i \in U$ we have $c_i = \Pi_U(c_i)$. Hence the above equals 
	$$ \Pi_\X \left(\Pi_U(y_1) - \eta   \frac{    \Pi_U(c_1) + \ldots + \Pi_U(c_n)}{\sqrt n} \right)   = \Pi_\X \circ \Pi_U \left(y_1 - \eta   \frac{    c_1 + \ldots + c_n}{\sqrt n} \right)$$
	where we have again used linearity of $\Pi_U$.
	Use Lemma \ref{ort} to remove the $\Pi_U$ from the above and get $$ x_{n+1} = \Pi_\X   \left(y_1 - \eta   \frac{    c_1 + \ldots + c_n}{\sqrt n} \right).$$  
	These are just the actions given $c_1,c_2,\ldots \ $ as required.
\end{proof} Lemma  \ref{ort2} is enough to strengthen Theorems \ref{worstcase} and \ref{T2}  to replace the constants $L,R$ with those obtained from the projected cost vectors. Since projection is nonexpansive the new constants are smaller.

\begin{theorem} \normalfont \label{projEuclidean} Suppose the domain has direction $U$. Given cost vectors $b_1,b_2,\ldots $  Theorem \ref{worstcase} holds with $L$ replaced with $ \wt L = \sup \big \{\|\Pi_U(b_i)\|: i \le N \big \} $. Given i.i.d cost vectors $a_1,a_2,\ldots$ with $\Ex[a_n]=a$ the bounds in Theorem  \ref{T2}  hold with $L$ and $R$ replaced with $ \wt L = \sup \big \{\|\Pi_U(a_i)\|: i \le N\big \} $ and $\wt R = \sup \big \{\|\Pi_U(a_i - a)\|: i \le N \big \}$.  
\end{theorem}

Lemmas  \ref{natbound} and  \ref{ort2} together let us replace the Euclidean bounds $L$ in Theorem \ref{worstcase}   with the intrinsic bound $ L_\infty/W$. Simplify to get the following.

\begin{theorem} \normalfont\label{worstcase2}
Let the domain $\X \subset \RR^d$ have diameter $D$ and width $W$.  Suppose the cost vectors $b_1,b_2,\ldots$ have   \mbox{$|b_n \cdot (x-y)| \le L_\infty$}  for all   $x,y \in \X$. Then Algorithm 1 with  domain $\P$ and $y_1 \in \P$  and parameter $\eta = DW /2L_\infty$ gives  regret bound  $\ds 
	\sum_{i=1}^N b_i \cdot (x_{i} - y^*) \le   \frac{3L_\infty D}{W} \sqrt {N}$. 
\end{theorem}
Likewise replace $L$ and $R$ with  in $L_\infty/W$ and $R_\infty/W$ in  Theorem \ref{T2} and simplify to get the following.

\begin{theorem} \normalfont\label{T3}   Let the domain $\P \subset \RR^d$ be a polytope with diameter $D$ and width $W$. Suppose the cost vectors $a_1,a_2,\ldots$ are i.i.d with  $\Ex[a_i] = a$ and satisfy the intrinsic bounds \eqref{nat}. Suppose we run Algorithm 1 with domain $\P$ and parameter $\eta=D/2L$ and starting point $y_1 \in \P$.  Then the pseudo-regret satisfies
	
	\begin{align*}
	\Ex \left [\sum_{i=1}^{\infty}   a \cdot( x_i-v^*)   \right ]  \le   \frac{L_\infty D}{W} + \left( 31 L_\infty\left (  2L_\infty  + \! \sqrt{ \frac{\pi}{2} }R_\infty \! \right )      \!+ \!    15 R^2_\infty\right) \frac{D^2}{W^2 \DD}    = O \left( \frac{L^2_\infty D^2}{W^2\DD}\right).%\label{secondbound1}
	\end{align*}
	
	%\begin{align*}
	% 5LD\left( 48\frac{L D }{  \DD }    + \frac{ \DD}{32 LD  }\right ) + \frac{R^2 D^2}{40} \left( \frac{1}{\DD_d } +    \frac{2 }{\DD}  \right) 
	%\end{align*}
	
\end{theorem}

\section{Examples with Intrinsic Bounds}
  
\noindent Here we examine the  polytopes  from Section 3 under the intrinsic  bounds (\ref{nat}) on the cost vectors. 
In Table 2 the columns $D,W$ and $V$ are the diameter, width and number of vertices. The values for $D$ are exact except for $\P(d)$ and $\P_\pm(d)$ where they are exact as $d \to \infty$. The values for $W$ are exact for the simplex and cube and are lower bounds for $\B(n), \P(d),\P_\pm(d)$. See Appendix A for discussion. The antagonistic and i.i.d bounds for Gradient Descent come from Theorems  \ref{worstcase2} and \ref{T3} respectively. The Hedge bounds refer to the lifting procedure detailed in Section 3 and come from  \cite{OptimalHedge}.

\begin{table}[!h]\def\arraystretch{1.2} \centering 
	\begin{tabular}{l|c|c|c|rrr}
	
\multicolumn{1}{c}{ }& 
\multicolumn{3}{c}{Dimensions}&
\multicolumn{1}{r}{Algorithm}&
\multicolumn{1}{c}{Antagonistic regret}&
\multicolumn{1}{c}{i.i.d pseudo-regret\vspace{1mm}}  \\ \hline  
Polytope               & $D^2$ & $V$ & $W^2$ & Gradient Descent & $ \ds \frac{L_\infty D }{  W} \sqrt N   $ & $\ds \frac{D^2 L_\infty^2  }{W^2 \DD^2} \rule{0pt}{4.5ex} $   \\ [5pt]  
		&   &   &    & Hedge & $ L_\infty \sqrt{ \log (V)  N }  $ &  $\ds \frac{L_\infty^2  \log (V)}{\DD} $   \\[7pt] \hline
		$d$-Simplex           & $ 2$ & $d$ & $  \ds \frac{2}{d-1} \rule{0pt}{4ex}   $ & Gradient Descent  &$ L _\infty\sqrt {dN}   $& $ \ds  L^2_\infty d/\DD  $  \\[5pt]
		&&&& Hedge & $ L_\infty \sqrt { \log (d)N} $ &  $\ds  L^2 _\infty \log (d) /\DD $ \\  [5pt]
		$d$-Cube              & $4 d$ &  $2^d$& $4$ & Gradient Descent & $L_\infty \sqrt {dN}   $&  $ L^2_\infty d /\DD $   \\
		&&&&Hedge&$L_\infty \sqrt {dN}$ &   $   L^2_\infty d /\DD  $ \\[5pt] 
		$\B(n)$  & $2n$  &  $n!$& $  \ds \frac{4}{ n-1 } $ &   Gradient Descent&$L_\infty  n \sqrt {N} $& $ L^2_\infty n^2 /\DD $ \\ 
		&&&& Hedge &$L_\infty \sqrt {n \log (n)N}$&  $  L^2_\infty n \log (n) / \DD $ \\  [5pt]
		$\P(d)$ &  $\ds \frac{d^3 }{3}$ &  $d!$ & $ \ds \frac{5d^2}{6}$ &   Gradient Descent &$L_\infty \sqrt {dN}$ & $ L^2_\infty d /\DD $     \\ 
		&&&& Hedge &$L_\infty \sqrt {d \log(d)N}$&  $  L^2_\infty d \log (d) / \DD $ \\  [5pt]
		$\P_\pm (d)$  &  $ \ds \frac{2  d^3  }{3}$ & $d^2 d!$  & $ \ds \frac{4d^2}{3} $ &  Gradient Descent  & $ L_\infty \sqrt {dN}  $ &$ L^2_\infty d /\DD $  \\
		&&&& Hedge &    $ L_\infty \sqrt {d \log(d)N}  $& $L^2_\infty d\log (d) /\DD  $ \\[5pt]  \hline 
	\end{tabular} 
	
	\captionsetup{format=hang}  
	\centering 
	
	\caption{\label{table}\openup .2em Comparison of Gradient Descent and Lifted Hedge under  intrinsic bounds on cost vectors $|a_n \cdot(x-y)| \le L_\infty$ and $|(a_n-a) \cdot(x-y)| \le R_\infty$ for all $x,y \in \P$. }
\end{table}   
\subsection{Discussion}

It is no surprise Hedge scales better with dimension than Gradient Descent on the simplex. Unfortunately the simplex is also the only example where Hedge is feasible in high dimensions. For the cube the Hedge and Gradient Descent bounds have the same order. However only   Gradient Descent  is feasible as the number of vertices increases.

For  the Birkhoff polytope,  Gradient Descent has an extra factor of  $\sqrt {n/\log (n)}  $ compared to Hedge. For example   $n=10$ gives  $\sqrt {n/\log (n)} \simeq 2.084 \ldots$ and the bound is roughly double. On the other hand  there are $n!=3628800$ vertices and running Hedge is computationally unfeasible. For comparison the main cost of Gradient Descent is projecting onto the domain. This can be done  using Franke-Wolfe (\cite{RevisitingFrankeWolfe} at cost $O(n^3)$. If we are satisfied with an  approximately feasible  point we can instead use Lagrange multipliers  with cost $O(n^2)$.

For the signed permutahedron the Gradient Descent bounds are slightly better than Hedge. The extra $\sqrt{\log(d)}$ factor suggests the permutahedron is {\it rounder} than the Birkhoff polytope. Unfortunately this same roundness means the polytope has $2^n$ facets. This makes the cost of projecting using Franke-Wolfe or Lagrange multipliers prohibitively large. Fortunately there exist more sophisticated methods (\cite{PProject1,PProject2})that use the polytope's structure to project with cost only $O(n)$.

%\begin{center}
%\begin{table}[!h]\def\arraystretch{1.5}
%\begin{tabular}{l|c|c|c|c|l|llll}
%Poltope               & $D^2$ & $W^2$ & $V$ & SBG & EXP3 &  &  &  \\  \hline
%$S_d$           & $ 2$ & $\ge   \frac{2}{d-1}   $ & $d$ & $O(L \sqrt {dN} ) \ O(L^2 d/\DD )$ &  &  &  &  \\ 
%$[-1,1]^d$              & $4 d$ &  $4$& $2^d$ & $O(L \sqrt {dN} ) \ O(L^2 d/\DD )$ &  &  &  &  \\ 
%$\B(n)$ & $n$  &  $\ge  \frac{4}{ n-1 } $& $n!$ &  &  &  &  &  \\ 
%$\P(d)$    &  $\frac{d(d^2-1)}{3}$ &  $\ge \frac{5d^2}{6}$ & $n!$ &  &  &  &  &  \\ 
%$\P_{\pm}(d)$  &  $ \frac{2  d(d+1)(2d+1)}{3}$ & $\ge \frac{4d^2}{3} $  & $2^d %d!$ &  &  &  & 
%\end{tabular}
%\end{table}
%\end{center}

 \section{Computational Cost}  The advantage of Gradient Descent over Hedge is the action vectors have length $d$ rather than $V$. For example the cube of dimension $d$ has $V = 2^d$ vertices and the cost of computing all the Hedge components  grows exponentially. 
 For comparison the most expensive part of Gradient Descent is projecting onto the domain. 
 
 Consider the example polytopes in Tables 1 and 2.  For the cube projection is straightforward, simply project each of the $d$ components onto the closed interval. For more complex polytopes the most efficient methods scale with the number of facets rather than vertices. For example we can use Lagrange Multipliers to find an approximate projection, taking one multiplier for each facet of the domain. 
 For the Birkhoff polytope (see Example 1) we can count the facets using the  formulation  
    $$\B(n) = \left\{x \in \RR^{n \times n}: x^i_j \ge 0 \text{ and }   \sum_{k=1}^d x^i_k \text{ for all }i,j \le d \right\}.$$ 
 There is one facet per entry of the matrix  and one facet per row and column. Hence the polytope has $F =  n^2 + 2n = O(n^2)$ facets and the Lagrange iteration has cost $O(n^2) = O(d)$. The downside is our action might lie outsite the domain $-$ though it tends towards a point in the domain with further iterations. The Frank-Wolfe Method has the advantage that it  returns a feasible point, at the cost of solving a sequence of linear problems on the polytope. For the Birkhoff polytope (\cite{RevisitingFrankeWolfe}) the iteration has cost $O(n^3)$. 
 
 The permutahedron (see Example 1) is resistent to such methods as it has $2^d$ faces, one for each subset of $\{1,2,\ldots, d\}$. Fortunately optimisation on the permutahedron can be reduced to optimisation on Birkhoff using the method of extended  formulation (\cite{ conforti2010extended,kaibel2011extended,goemans2015smallest,rahmanian2016online}). There exists a linear surjection $\phi:\B(d) \to \P(d)$ given by $\phi(x)  = \sum_{j=1}^d (j x^1_j,\ldots, j x^d_j ) $.
 Hence given a convex function $f: \P(d) \to \RR$ we can instead  optimise $\phi \circ f : \B(d) \to \RR$  using the methods of the previous paragraph.
\subsection{Barrier Functions}
After Hedge and Gradient Descent, the most familiar online optimisation algorithms use barrier functions.  In Section 3 we  generalise Hedge from the simplex to general polytopes, by replacing the polytope with a simplex with the same number of vertices. The Barrier algorithm is an alternate generalisation where facets take precedence rather than vertices. For example vanilla Hedge on the simplex is a special case of Barrier for  all  $\phi_j(x)=x(j)$ and   $\eta_j =1$. We suspect the Barrier algorithm scales better computationally since there is one barrier per facet rather than per vertex.

For simplicity suppose the domain $\P \subset \RR^d$ has nonvoid interior and $F$ facets. We can write $\P = \{x \in \RR^d: \phi_j (x)  \ge 0 \text{ for } i =1,2,\ldots ,F\}$ for some affine functions $\phi_j : \RR^d  \to \RR$ corresponding to the facets. This gives the following algorithm.

\begin{algorithm}[!h]\caption{Anytime Barrier Function Algorithm}
	\DontPrintSemicolon % Some LaTeX compilers require you to use \dontprintsemicolon instead
	\KwData{Polytope action set $\P \subset \RR^d$. Base Point $x_1 \in \P$. Affine functions $\phi_1,\ldots, \phi_F  : \RR^d \to \RR$. Parameters $\eta_1,\ldots, \eta_F >0$. \;}
	%Regularisation function $R(w) = \sum_{i=0}^d w[i] \log w[i]$ on $\W$
	select action $x_1  \in \P$\;
	\text{pay cost} $a_1 \cdot x_1$	\;
	\SetKwBlock{Loop}{Loop}{EndLoop}
	\For{$n=2,3,\ldots$}{
		
		\text{recieve} $a_{n-1}$\;
		
		\text{select action} $\ds x_n = \substack{\am \\ x \in \P} \left ( \sum_{i=1}^F \eta_j  \phi_j (x) \log ( \phi_j(x)) + \frac{1}{\sqrt{n-1}}\sum_{i=1}^{n-1} a_j \cdot x \right)$\;
		\text{pay cost} $a_n \cdot x_n$	
	}
\end{algorithm}
\noindent  
It is an open problem whether the Barrier algorithm is universal, and how  the  computational  cost and regret bounds compare to Gradient Descent and Hedge. We suspect Barrier gets the best results on the cube. By symmetry we should select  $\eta_1 = \eta_2 = \ldots = \eta_F$. For more general problems we suspect it is important to tune   $\eta_j$ based on the dimensions of the polytope.

\subsection{Higher-Order Estimates}

\noindent It was recently proved  (\cite{FTLBall, GDStronglyCurved})  that Gradient Descent is universal on a strongly convex domain, with $O(\sqrt{\log N})$ pseudo-regret in the i.i.d setting.
One shortcoming of our Theorems \ref{T2} and \ref{T3} is they do not recover the $\log(N)$ bounds as the polytope approaches a strongly convex domain. 

For example if the domain is a regular $n$-gon we would like the pseudo-regret bound to recover the $\log(N)$ bound for the unit ball as $n \to \infty$. This does not happen however. Instead the $O\left(L^2 D^2/\DD\right )$ bound from Theorem \ref{T2} goes to infinity as    $\DD \to 0$ and $D,W \to 2$.  To fix this one idea is to replace the estimate from Lemma \ref{conebound} $$ \theta_v = \frac{1/2}{1+D^2\|a\|^2/\DD_v^2}-1$$

with a more sophisticated quantity. For example a second-order estimate in terms of the dihedral angles and side-lengths of the polytope. For example small sides and large angles means the polytope {\it curves} upwards from its lowest point at a faster rate. This should give better estimates than Lemma \ref{conebound}. On a similar note we predict that explicitly computing $\theta_v$ for the polytopes in Tables 1 and 2 will give better performance bounds than simply plugging the dimensions of the polytopes into Theorems \ref{T2} and \ref{T3}. This may yield  better choices of the hyperparameter $\eta$ in Algorithm 1.

\section*{Acknowledgements}
This work was supported by Science Foundation Ireland grant 16/IA/4610.

\bibliography{bibfile} 
\bibliographystyle{plainnat} 

\section*{Appendix A: Dimensions of Polytopes}

\noindent Here we derive bounds for the width and diameter of the polytopes in Tables 1 and 2. To our knowledge  the widths of Examples  3-5    do not appear in the literature  at all.   Thanks to David E \cite{Wide} for suggesting the probabilistic counting trick used in those examples. 

Recall Definition \ref{widthdef} of the width:  Let $\X \subset \RR^d$ be convex    with direction $U = \{t(x-y):x,y \in \X,t \in \RR \}$. For each $\ell \in \RR^d$ let $W_\ell$ be the length of the interval $\{ \ell \cdot x: x \in \X\}$. The width of $\X$ is defined as $W = \inf \{ W_\ell: \ell \in U,  \|\ell\|=1 \}$. \\

\noindent {\bf Example 1 }We claim   an $m$-dimensional cuboid $\prod _{j=1}^m [\AA_j,\BB_j]$ embedded in $\RR^d$ has width  $W= \min_j|\BB_j-\AA_j|$. Since the affine hull is linearly isometric with $\RR^m$ we can assume $m=d$ and the embedding is the standard. We can also translate the cube to put the centre of mass at the origin and hence assume the cuboid is $\prod _{j=1}^d [-\G_j,\G_j]$ for $\G_j = \frac{1}{2}|\BB_j-\AA_j|$. Consider  $\ell \cdot(x-y) \ge 2\G_j  $ for  $x,y\in \prod _{j=1}^m [-\G_j,\G_j]$. By permuting the coordinates we can assume $\ell=(l_1,\ldots, l_n, -l_{n+1}, \ldots, -l_d)$ for all $l_j \ge 0$. Then $\ell$ is maximised (minimised) over the cube at $\pm p$ where $p= (\G_1 ,\ldots, \G_n , -\G_{n+1} , \ldots, -\G_d )$ has exactly $n$ positive entries. Hence $W_\ell =   2 \sum_{j=1}^d l_j \G_j \ge 2 \min_j \G_j \sum_{k=1}^d l_k =   2 \min_j \G_j \|\ell \|_1 \ge 2 \min_j \G_j \|\ell \|_2 =  2 \min_j \G_j =   \min_j |\BB_j-\AA_j|  $. \\
 
\noindent {\bf Example 1.5 } The cube has diameter $2\sqrt {\sum _{j=1}^m |\BB_j - \AA_j|^2}$ due to the following lemma.\\

\begin{lemma}\normalfont Each polytope $\P$ has vertices $u,v$  with $\|u-v\| = \max \{\|x-y\|:x,y \in \P\}$.
\end{lemma}

\begin{proof}The domain is compact so  $\max \{\|x-y\|:x,y \in \P\}=   \|x-y\|$ for some $x,y \in \P$. We claim $D = \|x-v\|$ for some vertex $v$. For write $y = \sum_{i=1}^V \ll_j  v_j$ as a convex combination of the vertices. Then we have 
\begin{align*}
 \|x-y\| = \bigg \| x-\sum_{i=1}^V \ll_j  v_j\bigg \| =\bigg \| \sum_{i=1}^V \ll_j  x-\sum_{i=1}^V \ll_j  v_j \bigg  \| \le \sum_{i=1}^V \ll_j \|   x - v_j   \| \le \max_{i \le V} \|x-v_j\|.
\end{align*} Thus the max is achieved for $y$ a vertex. Likewise it is achieved for $x$ a vertex. 
\end{proof}

\noindent {\bf Example 2 } The  $d$-simplex $\S$ has width   asymptotically equal to $2/\sqrt{ d }$.  The authors were unable to find a modern proof of this fact. The standard proof  seems to be \cite{SimplexWidth}. Unfortunately the author refers to {\it fundamental properties of convex sets} that were perhaps more well-known at the time. Therefore we refer to (1.9) of \cite{gritzmann1992inner}.

\begin{theorem} \normalfont
 Suppose the polytope $\P \subset \RR^N$ has non-empty interior and width $W$.  There exists a unit vector $\ell \in \RR^N$ and faces $A,B$ of $\P$ and points $a \in A$ and $b \in B$ such that  $$(1) \  \{\ell \cdot x: x \in \P\} \text{ has length } W  \qquad \qquad  (2)\  \dim (A-B) = N-1.$$ In particular $\dim A + \dim B \ge N-1$.
\end{theorem}

Note the given proof has a typographical error. For the proof to work we must use $(1.8)$ of that paper to take the points $q_{\pm }$ of the form $q_\pm = \pm q$ for some $q \in \frac{\P-\P}{2}$. Then $\ell \cdot (q_+  -q_- )$ is the width of $\frac{\P-\P}{2}$. 
As stated in the proof $q_+=q$ is in the relative interior of some facet $F$  of the symmetric polytope $\frac{\P-\P}{2}$ and so $q_-=-q$ is in the relative interior of the facet $-F$.  Theorem 3.1.2 of \cite{weibel2007minkowski} says $\P$ has  faces $A,B$   with $F = \frac{G-H}{2}$. Hence $N-1 = \dim F = \dim \big( \frac{A-B}{2}\big )= \dim(A-B) \le \dim A + \dim B$. 
Choose $a \in A$ and $b \in B$ with $\frac{a-b}{2} = q$.  It follows from   $(1.4)$ and $(1.5)$ of \cite{gritzmann1992inner} that $\P$ and $\frac{\P-\P}{2}$ have the same width. Thus $W = \ell \cdot (q_+  -q_- ) = 2 \, \ell \cdot q = 2\, \ell \cdot \big (\frac{a-b}{2} \big ) = \ell \cdot (a-b)$ as required.

%In fact the points $q_{\pm}$ obtained using $(1.8)$ can be taken of the form $q_\pm = \pm q$ for some $q \in \frac{\P-\P}{2}$. Then as stated $\ell \cdot (q_+-q_-)$ is the width of $\frac{\P-\P}{2}$ which by (1.4) and (1.5) has the same width as $\P$. Since $q$ is an interior point of some facet $A=\frac{F-G}{2}$ we can write $q = \frac{f - g}{2}$  and $W = \ell \cdot (q_+-q_-) = 2 \ell \cdot q_+ = \ell \cdot (f-g)$. 
 
  To use the theorem identify the affine hull of the $d$-simplex with $\RR^N$ for $N=d-1$. 
  The theorem gives faces $A$ and $B$ with $\dim A + \dim B \ge d-2$. The face $A$ contains $I$ of the vertices and $B$ contains $J$ of the vertices.  Since the dimension of a face is one less than the number of vertices it contains we have  $I+J -2 \ge d-2 $ and $I + J \ge d$ and so  $A \cup B$ contains all the vertices. Without loss of generality $A$ contains some $v_1,\ldots, v_n$ and $B$ contains $v_{n+1},\ldots, v_d$. Thus $\ell = (a,\ldots, a,b,\ldots,b)$ for some  $a \le b$. 
Since   $\ell \cdot \1=0$ we see $\ell$ is a scalar multiple of $u = \left(-\frac{1}{n} ,\ldots,-\frac{1}{n}, \frac{1}{d-n},\ldots, \frac{1}{d-n}\right)$. Since $\|u\|^2 = \frac{1}{n}+\frac{1}{d-n} = \frac{d}{n(d-n)}$ and $\|\ell\|^2=1$ we get $\ell = \sqrt{\frac{n(d-n)}{d}}\left(-\frac{1}{n} ,\ldots,-\frac{1}{n}, \frac{1}{d-n},\ldots, \frac{1}{d-n}\right)$ and so $$W = b-a = \sqrt{\frac{n(d-n)}{d}}\left( \frac{1}{d-n} + \frac{1}{n}\right)=\sqrt{\frac{n(d-n)}{d}}\frac{d}{n(d-n)} =\sqrt{\frac{d}{n(d-n)}}.$$ The function is decreasing for $n \le d/2$ and increasing thereafter. Hence the minimum value occurs at $n=d/2$ for $d$ even and $n=\lceil d/2 \rceil,\lfloor d/2 \rfloor$ for $d$ odd. In the first case we get $W =  2 / \sqrt d $ and in the second $W =  \frac{2}{\sqrt d } \sqrt{ \frac{ d/2 }{\lceil d/2 \rceil }}\sqrt{ \frac{ d/2 }{ \lfloor d/2 \rfloor}}$ which equals $2/\sqrt d$ as $d \to \infty$.\\

\noindent {\bf Example 2.5 } The distance between any two vertices of the simplex is $\sqrt 2$. Thus the diameter is $D = \sqrt 2$ independent of dimension. \\
 
 \noindent {\bf Example 3 }  The Birkhoff Polytope $\B$ is the set of nonnegative $n\times n$ matrices with all row and column sums equal to $1$. Equivalently the convex hull of the $n!$ permutation matrices. Identify each permutation $\si \in \Sigma_n$ with the corresponding matrix. We claim $W \ge  2/\sqrt{n-1}$. To prove this let $\ell \in \RR^{n \times n}$ have $\|\ell\|^2=1$ and all $\Sum \ell^i_j= \sum_{j=1}^n \ell^i_j=0$.
 
 The proof uses a  probabilistic counting trick. Let $\si \in S _n$ be a uniformly-chosen permutation matrix and consider the random variables $X = \ell \cdot \si $. We can write 
 
 \begin{align*}
   \sum_{\si \in S_n} (\ell \cdot \si )^2 = \sum_{\si \in S_n} \left (\sum_j \ell^i_{\si(i)} \right )^2  = \sum_{\si \in S_n} \left ( \sum_j  \big (\ell^i_{\si(i)} \big )^2  +   \sum_{i \ne j}   \ell^i_{\si(i)} \ell^j_{\si(j)} \right )
 \end{align*} 
 
 For each pair $(i,j)$ the term $(\ell^i_{j} \big )^2 $ appears in the expansion $(n-1)!$ times, since this is the number of permutations with $\si(i)=j$. For each tuple $(i,j,a,b)$ with $i \ne j$ and $a \ne b$ the term  $\ell^i_{a} \ell^j_{b}$ appears $(n-2)!$ times since there are $(n-2)!$   permutations with $\si(i)=a$ and $\si(j)=b$. Hence the above equals
 \begin{align*}
   (n-1)! \sum_{i,j \le n} (\ell^i_{j} \big )^2 + (n-2)! \sum_{\substack{i \ne j\\a \ne b} } \ell^i_{a} \ell^j_{b} = (n-1)!  + (n-2)! \sum_{ i,a} \ell^i_{a}  \sum_{ j \ne i } \sum_{ b \ne a  }\ell^j_{b}
 \end{align*} 
 
 where we have used $\|\ell\|^2=1$ to simplify the first term. For the second term, since row $j$ and column $b$ sum to zero we have
 
 \begin{align*}
    \sum_{ i,a} \ell^i_{a}  \sum_{ j \ne i } \sum_{ b \ne a  }\ell^j_{b} = \sum_{ i,a} \ell^i_{a}  \sum_{ j \ne i } (-\ell^j_{a}) = - \sum_{ i,a} \ell^i_{a}  \sum_{ j \ne i } \ell^j_{a} =   - \sum_{ i,a} \ell^i_{a}(-\ell^i_{a}) = \sum_{ i,a} (\ell^i_a)^2 =1.
 \end{align*} 

 We conclude
 \begin{align*}
   \sum_{\si \in S_n} (\ell \cdot \si )^2 = (n-1)! + (n-2)! = n(n-2)!
 \end{align*} 

 Hence $X = \ell \cdot \si$ has variance $n(n-2)!/n! = 1/(n-1)$ and standard deviation $1/\sqrt{n-1}$. Popoviciu's inequality says the standard deviation is at most $\frac{\max X - \min X}{2}$. From this we get $\max X - \min X \ge 2/\sqrt{n-1}$ as required.\\

\noindent {\bf Example 3.5 } The diameter is achieved for any pair of permutation matrices  with no nonzero entries in common. Thus we have $D^2 = 2 n$. \\

 \noindent {\bf Example 4} The permutahedron $\P$ is the set of vectors $x \in \RR^d$ with entries $\{x_1,\ldots, x_d\} = \{1,2,\ldots, d\}$.  Equivalently $\P$ is the convex hull of $\{(\si(1),\ldots, \si(d)): \si \in S_d \}$. Identify the permutation $\si \in S_d$ with the vector $(\si(1),\ldots, \si(d)) \in S_d$. We claim 
 
 $$W\ge  \sqrt{\frac{5d^2 +8d +4 }{6}} \qquad \qquad   \liminf_{d \to \infty} \frac{|P |}{d} \ge \sqrt{5/6} $$ We use the same  variance trick with $X=\ell \cdot \si$

 \begin{align*}
   \sum_{\si \in S_n} (\ell \cdot \si )^2 = \sum_{\si \in S_n} \left (\sum_j \si(i)\ell_j  \right )^2 = \sum_{\si \in S_n} \left ( \sum_j  \si(i)^2 \ell_j ^2  +   \sum_{i \ne j}   \si(i)\si(j)\ell_j \ell_j \right ) 
 \end{align*} 
 
 For each $\si$ we have $\si(1)^2 + \ldots + \si(d)^2 = \frac{d(d+1)(2d+1)}{6}$. Thus the sum of coefficients in the first sum is  $$\sum_{\si \in S_n}  \sum_j  \si(i)^2 = \frac{d(d+1)(2d+1)}{6}d! =  \frac{d(2d+1)(d+1)!}{6}.$$ By symmetry each $\ell_j^2$ appears in the expansion with multiplicity $\frac{(2d+1)(d+1)!}{6}$. The sum of coefficients in the second sum is 
 
 \begin{align*}\sum_{\si \in S_n}  \sum_{i \ne j}   \si(i)\si(j)  &= \frac{1}{2}\sum_{\si \in S_n}  \Big ( \big (\si(1) + \ldots + \si(d)\big)^2 -  \si(1)^2 - \ldots - \si(d)^2\Big )\\ 
 &= \frac{1}{2}\sum_{\si \in S_n}  \left ( \left ( \frac{d(d+1)}{2}\right)^2   - \frac{d(d+1)(2d+1)}{6}\right )\\
 &= \frac{1}{2}\sum_{\si \in S_n}  \left (   \frac{d^2(d+1)^2}{2}    - \frac{d(d+1)(2d+1)}{6}\right )\\
 &= \frac{1}{2}\sum_{\si \in S_n}      \frac{3d^4 +4d^3 -d}{6}  =  \frac{(3d^4 +4d^3 -d)d!}{12}        
 \end{align*} 
  
 Since there are $d(d-1)$ choices for the pair $(i,j)$ with $i \ne j$ we have by symmetry that each $\ell_j \ell_j$ appears in the expansion with multiplicity $\frac{(3d^3 +4d^2 -1)(d-1)!}{12}$. Thus we have shown

 \begin{align*}
   \sum_{\si \in S_n} (\ell \cdot \si )^2 &=  \frac{(2d+1)(d+1)!}{6} \Sumd \ell_j^2 + \frac{(3d^3 +4d^2 -1)(d-1)!}{12} \sum_{i \ne j} \ell_j \ell_j%\\
   %& =  \frac{(2d+1)(d+1)!}{6}+ \frac{(3d^3 +4d^2 -1)(d-1)!}{12} \sum_{i \ne j} \ell_j \ell_j\\
 \end{align*} 
 
 To simplify the first term recall $\Sumd \ell_j^2  = \|\ell\|^2=1$. For the second write 
 \begin{align*}
    \sum_{i \ne j} \ell_j \ell_j = \frac{1}{2}\Big ( \big (\ell_1 + \ldots + \ell_d\big)^2 -  \ell_1^2 - \ldots - \ell_d^2\Big ) - \frac{\|\ell\|^2}{2} = - \frac{1}{2}.
 \end{align*} 
 
 Thus we have 
 
 \begin{align*}
   \sum_{\si \in S_n} (\ell \cdot \si )^2 &=  \frac{(2d+1)(d+1)!}{6} - \frac{(3d^3 +4d^2 -1)(d-1)!}{24} \\
   & \ge   \frac{(2d+1) (d+1)!}{6} - \frac{(3d^2 +4d)d!}{24} =   \frac{4(2d+1)(d+1)   -(3d^2 +4d) }{24} d!
 \end{align*}
 
 and the variance is

 \begin{align*}
     \frac{4(2d+1)(d+1)   -(3d^2 +4d) }{24}  =  \frac{8d^2 +12d+4  -(3d^2 +4d) }{24} = \frac{5d^2 +8d +4 }{24}
 \end{align*}
 
 and standard deviation 
 
 \begin{align*}
     \sqrt{\frac{5d^2 +8d +4 }{24}} = \frac{1}{2}\sqrt{\frac{5d^2 +8d +4 }{6}} 
 \end{align*}
 
 Like before we see $\max X - \min X \ge  \sqrt{\frac{5d^2 +8d +4 }{6}} $ as required. For large $d$ the above is approximately $  \sqrt{5/6}\, d$.\\
 
\noindent {\bf Example 4.5 } We claim the diameter is achieved for the vertices $v = (1,2,\ldots, d)$ and $w =(d,d-1,\ldots, 1)$. For suppose $\si$ and $\mu$ are vertices. By symmetry we can assume $\mu $ is the identity. For some $m \le d$ we have $\si(m)=1$. Suppose $m \ne d$. We can write \begin{align*}
\|\si - \mu\|^2 & =   (1- \si(1))^2 + \ldots + (m- \si(m))^2 + \ldots + (d- \si(d))^2\\
& = \sum_{n=1}^d n^2 +\sum_{n=1}^d \si(n)^2 - 2 \sum_{n=1}^d n \si(n) =2\sum_{n=1}^d n^2 - 2 \sum_{n=1}^d n \si(n).\end{align*} 

The first term is independent of $\si$. Hence  to maximise $\|\si-\mu\|^2$ we must minimise  $\sum_{n=1}^d n \si(n)$. 
We will prove a more general statement. Suppose $0 < x_1 < x_2 <\ldots < x_d$ and $ y_1 > y_2 >\ldots > y_d>0$. We claim that $\sum_{n=1}^d x_n y_{\si(n)}$ is minimised when $\si$ is the identity. 

For a contradiction suppose $\si$ minimises but has $\si(i) > \si(j)$ for some $i < j$. Define $\tau$ by $\tau(1) = \si(2)$ and $\tau(2) = \si(1)$ and $\tau(n) =\si(n)$ otherwise. The difference is
\begin{align*}
 \sum_{n=1}^d x_n y_{\tau(n)} - \sum_{n=1}^d x_n y_{\si(n)} &= x_1 y_{\tau(1)} + x_2 y_{\tau(2)}- \big ( x_1 y_{\si(1)} + x_2 y_{\si(2)}\big)\\
 & = x_1 y_{\si(2)} + x_2 y_{\si(1)}-  x_1 y_{\si(1)} - x_2 y_{\si(2)} = (x_1-x_2) (y_{\si(2)} -y_{\si(1)} )
\end{align*}

Since $x_n$ are increasing we have  $x_1 - x_2 < 0$. Also since $\si(1) > \si(2)$ and $y_n$ are decreasing we have  $y_{\si(2)} -y_{\si(1)}>0$. Hence the right-hand-side is negative. That implies $\sum_{n=1}^d x_n y_{\tau(n)} < \sum_{n=1}^d x_n y_{\si(n)} $ and $\si$ is not a minimiser.
It follows $\|\si - \mu\|^2$ is maximised   for $\si = w$ and $\mu = v$. 

Now we claim $D^2 = \frac{d(d^2-1)}{3}$.
For even $d$ the we see \begin{align*}D^2 = \|u-w\|^2  = (d-1)^2 + (d-3)^2 + \ldots + 3^2 + 1^2 + 1^2 + 3^2 + \ldots + (d-1)^2 \\ 
 = 2 \big ( (d-1)^2 + (d-3)^2 + \ldots + 3^2 + 1^2 \big )\end{align*}
 is twice the sum of odd squares.
To compute this recall the sum of the first $d$ squares (\cite{Sums}) is $\frac{d(d+1)(2d+1)}{6}$. Hence the  sum of the first $d/2$ even squares is $$\ds  \sum_{n=1}^{d/2} (2n)^2 =   4 \sum_{n=1}^{d/2} n^2 =  4 \frac{\frac{d}{2}(\frac{d}{2}+1)(d+1)}{6} = \frac{d(d+2)(d+1)}{6}.$$
The sum of odd squares is the sum of all squares minus the sum of even squares and so equals
 
 $$ \frac{d(d+1)(2d+1)}{6}- \frac{d(d+2)(d+1)}{6} = \frac{d(d+1)(d-1)}{6} = \frac{d(d^2-1)}{6}  $$ 
 
 For odd $d$ we see $D^2$ is twice the sum $(d-1)^2 + (d-3)^2 + \ldots + 2^2$ of the first $\frac{d-1}{2}$ even squares. By the above it equals $\frac{(d-1)(d+1)d}{6}$ and so $D^2 = \frac{d(d^2-1)}{3}$ like before.\\

 \noindent {\bf Example 5} The signed permutahedron $\P_\pm$ is the convex hull of the vectors $ (\pm\si(1),\ldots, \pm\si(d))$ for all choices of signs and permutation $ \si \in S_d $. We claim
 $$|\P_\pm| \ge  2\sqrt{  \frac{2d^2 +3d +1}{6}}  \qquad \qquad   \liminf_{d \to \infty} \frac{|P_\pm |}{d} \ge \sqrt{8/6} $$ %We use the same  variance trick with $X=\ell \cdot \si$

 For $S^d = \{-1,1\}^d $ we can write 
 
 $$\P_\pm =\{ (s_1\si(1),\ldots, s_d\si(d)) : \si \in S_d, s \in S^d\}$$
 
 Write $s \si = (s_1\si(1),\ldots, s_d\si(d)) $ and consider the random variables $X = \ell \cdot s\si $.
  
 \begin{align*}
    \sum_{s \in S^d}\sum_{\si \in S_d} (\ell \cdot s \si )^2 =  \sum_{s \in S^d}\sum_{\si   \in S_d} \left (\sum_j \si(i)\ell_j  \right )^2 &=  \sum_{s \in S^d}\sum_{\si \in S_d} \left ( \sum_j  s_j ^2 \si(i)^2 \ell_j ^2  +   \sum_{i \ne j}   s_j s_j \si(i)\si(j)\ell_j \ell_j \right )\\
    &=  \sum_{s \in S^d}\sum_{\si \in S_d} \left ( \sum_j   \si(i)^2 \ell_j ^2  +   \sum_{i \ne j}   s_j s_j \si(i)\si(j)\ell_j \ell_j \right ) 
 \end{align*} 
 
 By symmetry the second part vanishes leaving

 \begin{align*}
      \sum_{s \in S^d}\sum_{\si \in S_d}  \sum_j   \si(i)^2 \ell_j ^2 = \sum_{s \in S^d}  \frac{(2d+1)(d+1)!}{6} \sum_j \ell_j^2 =2^d \frac{(2d+1)(d+1)!}{6}    
 \end{align*} 

 The first equality uses the argument from the previous example to compute the coefficients. The second equality uses $\|\ell\|^2=1$. Since $|S^d \times S_d| = 2^d d!$ the variance is $$ \frac{(2d+1)(d+1)}{6} = \frac{2d^2 +3d  +1}{6}$$

 Like before we see $$\max X - \min X \ge  2\sqrt{  \frac{2d^2 +3d +1}{6}} $$ as required. For large $d$ the above is approximately $  \sqrt{8/6}\, d$.\\
 
\noindent {\bf Example 5.5 } The diameter is achieved for some pair $v,w$ of vertices. Similar to Example 4.5 we see $\|v-w\|^2$ is maximised for $v = (1,2,\ldots, d)$ and $w =(-1,-2,\ldots, -d)$. Then $$ D^2 = \sum_{n=1}^d (2n)^2 = 4\sum_{n=1}^d n^2  = 4 \frac{d(d+1)(2d+1)}{6} =   \frac{2d(d+1)(2d+1)}{3} .$$

\section*{Appendix B: Probability}

\noindent The concentration result used to prove Theorem \ref{T2}  is the following. \vspace{2mm}

\noindent {\bf Theorem \ref{Pinelis}} Suppose the i.i.d sequence $X_1,X_2,\ldots$ takes values in $\RR^d$. Suppose each $\Ex[X_i] = 0$ and $\|X_i\| \le R$. Then for each $r \ge 0$ we have
	$$P \left( \Big \|\sum_{i=1}^n X_i \Big\| \ge  n r \right) \le 2\exp \left( -\frac{r^2}{2 R^2} n\right).$$ 
\vspace{2mm}

The above  is a special case of \cite{GoodAH} Theorem 3.5 about vector-valued martingales. For the definition of a martingale and martingale difference sequence see for example \cite{Bill} Section 35. We need only the fact that a mean-zero i.i.d sum defines a martingale.

\begin{theorem} \normalfont\label{Pinelis1} (Pinelis Theorem 3.5) Suppose the martingale $f_1,\ldots, f_n$ takes values in the $(2,D)$-smooth Banach space $(E,\|\cdot\|)$. Suppose we have $ \|f_1\|_\infty^2 + \sum_{i=2}^n \|f_i - f_{i-1}\|_\infty^2 \le b^2$ for some constant $b$. Then for all $t \ge 0$ we have
	$$P \left( \max\{\|f_1\|, \ldots, \|f_n\|\} \ge t \right) \le 2\exp \left( -\frac{t^2}{2D^2 b^2}\right).$$
\end{theorem}
 
To explain the notation  $\|f\|_\infty = \max\{\|f(x)\|: x \in \Omega\}$ is the $\sup$ norm taken over the probability space $\Omega$. The Banach space $(E, \|\cdot\|)$ is called $(2,D)$-smooth for $D \ge 0$ to mean $\|x+y\|^2 + \|x-y\|^2 \le 2\|x\|^2 + 2D^2\|y\|^2$ for all $x,y \in E$. This generalises the parallelogram law for $\RR^d$ with the Euclidean norm, where for $D=2$ the inequality becomes an equality.

To obtain Theorem \ref{Pinelis} from Theorem \ref{Pinelis1} take $f_i =X_1 + \ldots + X_i$ and $t = nr$. We have $ \|f_1\|_\infty^2 + \sum_{i=2}^n \|f_i - f_{i-1}\|_\infty^2 = \sum_{i=1}^n \|X_i\|^2 \le n R^2$ so we can take $b = \sqrt n R$ and simplify the right-hand exponent with $D=2$.

For real-valued martingales there exists a one-sided version of the above without the leading factor of $2$. For proof of the following see \cite{MIT2}.

\begin{theorem} \normalfont[Azuma-Hoeffding] Suppose $X_1,X_2,\ldots$ is a real-valued martingale difference sequence with each $|X_i| \le R$. Then for each $r\ge0$ we have $$P \left( \sum_{i=1}^n X_i  \ge  \sqrt n r \right) \le \exp \left( -\frac{r^2}{2R^2 }\right).$$
\end{theorem}The next lemma is used to bound the pseudo-regret in terms of the regret.

\begin{lemma}\label{martingale} \normalfont Let $a_1,a_2,\ldots$ be an i.i.d sequence of cost vectors with $\Ex[a_n] = a$ and let $x_1,x_2,\ldots$ be the actions of Algorithm 1. For $\ds x^* \in \arg \! \min_{x \in X} a \cdot x$ the random variables $X_i = (a-a_i) \cdot( x_i-x^*)$ define  a martingale difference sequence  with respect to the filtration generated by $a_1, a_2,\ldots $. 
\end{lemma}

\begin{proof}
 To prove $X_i$ is a martingale difference sequence we must show   $\Ex[X_n \,|\, a_1, \ldots, a_{n-1}] =0$. That means for each set $U = a_1^{-1}(U_1) \cap \ldots \cap a_{n-1}^{-1}(U_{n-1}) $ in the algebra generated by $a_1,a_2,\ldots a_{n-1}$ we have  $\int_U X_n dP =0$.
 To that end write each $B(i) = a_i^{-1}(U_i)$ and observe the indicator $\1_{B(i)}$ is a measurable function of $a_1,\ldots, a_{n-1}$. Now write
	
	\begin{align*} \int_U X_n dP = \int_U  (a-a_{n }) \cdot( x_n -x^*) dP = \int   (a-a_{n }) \cdot( x_n-x^*)  \1_{B(1)} \cdot \ldots \cdot \1_{B(n-1)} dP.
	\end{align*}
Recall $x_n $ is a function of $a_1,\ldots, a_{n-1}$. Since all $a_i$ are independent we can distribute to get 
	\begin{align*} \int_U  (a-a_{n }) \cdot( x_n -x^*) dP = \int   (a-a_{n}) dP \cdot \int ( x_n -x^*)  \1_{B(1)} \cdot \ldots \cdot \1_{B(n-1)} dP.
	\end{align*}
	Since $\Ex[a_n]=a$ the above is zero as required. 
\end{proof}
Next we apply the previous lemma.

\begin{lemma}\label{MartingaleBound} \normalfont Suppose we run Algorithm $1$ on domain $\X$ with diameter $D$. For each $M \in \NN$ we have
	\begin{align*}
	\Ex \left [\sum_{i=1}^{M} (a-a_i) \cdot( x_i-x^*)\right ] \le  \sqrt{ \frac {\pi}{2} }RD \sqrt M  
	\end{align*}
\end{lemma}

\begin{proof}Lemma \ref{martingale} says  $X_i = (a-a_i) \cdot( x_i-x^*)$ is a martingale difference sequence  with respect to   $a_1, a_2,\ldots $. Since $|X_i| = |(a-a_i) \cdot( x_i-x^*)| \le \| a-a_i \| \| x_i-x^*\| \le RD $  the Azuma-Hoeffding inequality says
	
	\begin{align*} P \left (\sum_{i=1}^{M}  (a-a_i) \cdot( x_i-x^*)  > t\right ) \le \exp\left(-\frac{t^2}{2  R^2 D^2 M }  \right).
	\end{align*}
	For $X = \sum_{i=1}^{M}  (a-a_i) \cdot( x_i-x^*)$ we can use Lemma   \ref{CDF} to bound the expectation:
	
	\begin{align*}\Ex \left [\sum_{i=1}^{M}  (a-a_i) \cdot( x_i-x^*)  \right ] & \le \int_0^\infty \exp\left(-\frac{t^2}{2   R^2 D^2  M}  \right)dt=  \frac{\sqrt \pi }{2} \sqrt{2 R^2 D^2 M} = \sqrt{ \frac {\pi}{2} }RD \sqrt M  
	\end{align*}
	
	where we have used (\cite{gaussian}) to evaluate the Gaussian integral. 
\end{proof}
The following fact about computing the expectation in terms of the CDF is well-known. But we were unable to find a suitably general proof in the literature.

\begin{lemma} \normalfont\label{CDF}
 Suppose $X$ is a real-valued random variable. Then \[\Ex[X] = \int_0^\infty P(X>x)dx - \int_{-\infty}^0 P(X\le x)dx.\] In particular  we have \[\Ex[X] \le  \int_0^\infty P(X>x)dx .\]

 \end{lemma}

\begin{proof} First assume $X$ is nonnegative. The second integrand vanishes away from $0$. Hence the second integral vanishes and we can write the first as 
\begin{align*}
\int_0^\infty P(X > x)dx = \int_0^\infty  \Ex_y  \mkern-5mu \left [ \1_{X(y) > x} (y) \right ] dx = \Ex_y  \mkern-5mu \left [ \int_0^\infty \1 _{X(y) > x} (y)dx \right].
\end{align*}
%where the expectation is takes with respect to $y$.
For fixed $y$ define the function $g(x) = \1 _{X(y) > x} (y)$. We have $g(x)= 1$ for all $x < X(y)$ and $g(x)= 0$ elsewhere. Since $X(y)$ is nonnegative that means $g(x)$ is the indicator function of $[0,X(y))$. It follows the inner integral equals $X(y)$ and we get $\int_0^\infty P(X > x)dx= \Ex_y [X(y)]=\Ex[X]$. Likewise we can define  $g(x) = \1 _{X(y) \ge x} (y)$ to get $\int_0^\infty P(X \ge x)dx = \Ex[X]$. 

For a general random variable  write $X=X^+ +X^-$ where $X^+$ takes only nonnegative values and $X^-$ only nonpositive values, and at each point one of $X^+$ or $X^-$ is zero. By linearity we have 
\begin{align*}
\Ex[X]  =\Ex[X^+] + \Ex[X^-] = \Ex[X^+] - \Ex[-X^-] =\int_0^\infty P(X^+>x)dx - \int_{0}^\infty P(-X^- \ge x)dx.
\end{align*}
where we have used the first paragraph for the nonnegative random variables $X^+$ and $-X^-$. To complete the proof recall $P(X^+>x)=P(X>x)$ since $X^+ > x$ occurs if and only if  $X >x$ for each $x >0$. For the second integral write  
\begin{align*}
 \int_0^\infty P(-X^- \ge x)dx = \int_0^\infty P(X^- \le -x)dx = \int_{- \infty}^0 P(X \le x)dx
\end{align*}  
since for each $x <0$ we have $X \le x$ if and only if $X^- \le x$. 
\end{proof}

 \begin{lemma}\normalfont  There exists an i.i.d opponent on the $2$-simplex such that the expected regret of every online algorithm against this opponent is $\Omega(\sqrt N)$.\label{bigEx}
 \end{lemma}

 \begin{proof}For simplicity identify the $2$-simplex with the interval $[-1,1]$. Let the costs be   $a_n = \pm 1$ each with probability $1/2$. The regret is $R_N = \sum_{i=1}^N a_i  \cdot  x_i   -  \sum_{i=1}^N a_i  \cdot y^*$ for $y^* = -1$ in case $\sum_{i=1}^N a_i \ge 0$ and $y^* = 1$ otherwise. Hence $  \sum_{i=1}^N a_i \cdot  y^*=  -  \big |\sum_{i=1}^N a_i \big| $ and   $\Ex[R_N] = \Ex \big [\sum_{i=1}^N a_i  \cdot  x_i \big]  +  \Ex \big | \sum_{i=1}^N a_i \big |$. Since each $x_i$ is a function of $a_1,\ldots, a_{i-1}$ it is independent of $a_i$ and we have $\Ex \big [\sum_{i=1}^N a_i \cdot   x_i \big]  = \sum_{i=1}^N \Ex[a_i]   \cdot \Ex[x_i ] = \sum_{i=1}^N 0 \cdot   \Ex[x_i ] = 0$. We conclude the expected regret is the absolute value  $ \Ex \big | \sum_{i=1}^N a_i \big |$ of a mean zero i.i.d sum which is $\Omega(\sqrt N)$ by the central limit theorem.
 \end{proof}

 %\begin{lemma}\normalfont  There exists an adversarial  opponent on the $2$-simplex such that the regret  of every online algorithm against this opponent is $\Omega(\sqrt N)$.\label{bigEx}
% \end{lemma}

 %\begin{proof}For simplicity identify the $2$-simplex with the interval $[-1,1]$. On turn $n$ the opponent sees $x_n$ and selects cost vectors $a_n = 1$ if $x_n \ge 1/\sqrt n$ and $a_n = -1$ otherwise. Similar to the above the regret is $R_N =  \sum_{i=1}^N a_i  \cdot  x_i    +   \big | \sum_{i=1}^N a_i \big |$. The first terms are positive if $x_n \ge 1/\sqrt n$ and at least  $-1/\sqrt n$ otherwise. Since $\SumN \frac{1}{\sqrt i} = \Theta(2\sqrt N)$ the first term satisfies $\sum_{i=1}^N a_i  \cdot  x_i  \ge -2 \sqrt N +O(1)$.   For $C_\pm = |\{n \in \NN: a_n = \pm\}|$  the second sum is $|C_+ - C_-|$. Consider two cases.
 
 %Case (a) We have $|C_+ - C_1| \ge 3\sqrt N$. In this case we get $R_N \ge -2\sqrt N +O(1)+ 3\sqrt N  \ge \Omega(\sqrt N)$.
 
 %Case (b) $|C_+ - C_1| \le 3\sqrt N$. Then $C_+ \ge N/2 - 3\sqrt N$ and $C_- \ge N/2 - 3\sqrt N$. Since the second sum is nonnegative we have
 
 %\begin{align*}
  %R_N \ge \sum_{i=1}^N a_i  \cdot  x_i = \sum_{i\in C_+}  a_i  \cdot  x_i +  \sum_{i\in C_-}  a_i  \cdot  x_i
 %\end{align*}

 %\end{proof}

 \begin{lemma}\normalfont  Suppose the cost vectors $a_1,a_2,\ldots$ are i.i.d with $a = \Ex[a_n] \ne 0$. Suppose the domain $\P$ is a polytope with distinct vertices $x^*_1,x^*_2 \in \ds \arg \!\min_{x \in \P} a \cdot x$. Then the Bernstein condition (4) from \cite{metagrad}  fails for all $\BB>0$.\label{bigEx1}
 \end{lemma}

 \begin{proof}Since the cost functions are linear  (4) simplifies to
 $ (x - x^*) ^ T \Ex \|a_n\|^2 (x - x^*) \le B ((x - x^*) ^ T a)^\beta $ for all $\ds x^* \in \arg \!\min_{x \in \P} a \cdot x $ and  $x \in \P$. Using our notation with  $x = x_1^*$ and $x^* = x_2^*$   we get $ \Ex \|a_n\|^2 \|x_1^*-x_2^*\|^2 \le B (a \cdot (x_1^* - x^*_2)  )^\beta $. The right-hand-side is zero by assumption. By the Jensen inequality the left-hand-side is at least $(\Ex[a_n] )^2\|x_1^*-x_2^*\|^2 $ which is nonzero by assumption. Hence the inequality fails.
 \end{proof}

\section*{Appendix C: Convex Geometry}
Here we prove two of the preliminary lemmas in Section 2.\vspace{3mm}

\noindent {\bf Lemma \ref{face}} {  \normalfont  Each  face $F$ of $\P$ is the convex hull of $F \cap \V$. } \vspace{3mm}
 
	\begin{proof} \cite{ConvexNotes} Theorem 4.7 says the vertices of a polytope are exactly the extreme points. Here   an extreme point  $x$ of  polytope $\P$ is one such that there are no $\ll \in (0,1)$ and  $y,z \in \P$  with $y,z \ne x$ and $x = \ll  y + (\ll -1)z$. In other words $x$ is not properly between any other two points of $\P$. 

Since $F$ is a polytope it is the convex hull of its extreme points. Hence we need only show each extreme point of $F$ is extreme in $\P$. It will then follow $F$ is the convex hull of some $\W \subset \V$. Clearly $\W \subset F \cap \V$. To see $\W = F \cap \V$ recall we assume no element of $\V$ is in the convex hull of the others. In particular no proper subset of $F \cap \V$ contains all of $F \cap \V$ in its convex hull.

To prove each extreme point of $F$ is extreme in $\P$ we will prove the contrapositive, that  each $x \in F$ that is non-extreme in $\P$ is also non-extreme in $F$. To that end write $F = \P \cap T$ for some tangent plane $T = \{y \in \RR^d: u \cdot y = u \cdot x\}$ for some $u \in \RR^d$. Since $T$ is tangent we can assume $u \cdot y \le u \cdot x$ for all $y \in \P$. 

Suppose $x \in F$ is not extreme in $\P$. Then   $x = \ll  y + (\ll -1)z$ for some $\ll \in (0,1)$ and  $y,z \in \P$  with $y,z \ne x$. By linearity we have $u \cdot x =  \ll  u \cdot y + (\ll -1) u \cdot z$. Since $y,z \in \P$ we have $u \cdot y  \le u \cdot x$ and $u \cdot z \le u \cdot x$. Hence the equality   $x = \ll  y + (\ll -1)z$ holds only if $u \cdot y = u \cdot x$ and $u \cdot z = u \cdot x$. In that case $y,z \in T$. Hence $y,z \in F$ and $x$ is not extreme in $F$. %This proves the first statement.
%To prove the second statement let $F = \P \cap T$ for the tangent plane $T = \{x \in \RR^d: a \cdot x = a_0\}$ for $a_0 = \inf\{a \cdot x: u \in \P\}$.
\end{proof}

\noindent {\bf Lemma \ref{conebound}}{    For each $v \notin \V -\V^*$ the quantity  $\ds \theta_v = \min \left \{ \frac{a \cdot u}{ \|a\|\|u\|} : u \in N_\P(v)\right \} $ satisfies $$ \theta_v \ge \frac{1/2}{1+D^2\|a\|^2/\DD_v^2}-1$$ Hence  $\theta_v >-1$ and the quantities $\phi_v = \theta_v+1$ are positive.}\vspace{3mm}

\begin{proof} By performing a rotation we can assume  $a = (\|a \|,0,\ldots,0)$. Note this does not change angles, suboptimality gaps, or Euclidean norms. The choice of coordinates gives $\DD_v = a  \cdot ( v-v^*)   = \|a \|(v(1) -v^*(1))$ and so $v^*(1) -v(1) = -\DD_v/\|a\|$.   First suppose $ u \in N_\P(v)$ has $u(1) \ge 0$ and consider the inequality
\begin{align} 
	\frac{a \cdot u}{\|a\|\|u\|}\ge \frac{1/2}{1+D^2\|a\|^2/\DD_v^2}-1.\label{conebound1}
	\end{align} 
	By assumption the left-hand-side is nonnegative. Since the right-hand-side is negative the above holds. 
	
	Now suppose $u(1) \le 0$.  Since $u \in N_\P(v)$ we know $\P$ is contained in the half-space $\{x \in \RR^d: u \cdot  x  \le u \cdot v \}$ and so  $u \cdot v^* \le u \cdot v$. 
	Expand the inequality componentwise and bring the $1$-components to the left to get  $
	u(1)\big (v^*(1) -v(1) \big) \le  \sum_{j=2}^d u(j) \big(v(j)-v^*(j) \big) $. Since $v^*(1) -v(1) = -\DD_v/\|a\|$ we get $
	-u(1) \frac{\DD_v}{\|a\|} \le   \sum_{j=2}^d u(j)    \big(v(j)-v^*(j)\big ) $. 
	The right-hand-side is the product of two $(d\!-\!1)$-dimensional vectors and Cauchy Schwarz gives 
	\begin{align}     \sum_{j=2}^d u(j)    \big(v(j)-v^*(j)\big ) \le  \sqrt{ \textstyle \sum_{j=2}^d u(j)^2  }   \sqrt{ \textstyle \sum_{j=2}^d (v(j)-v^*(j))^2  } \le  \sqrt{ \textstyle \sum_{j=2}^d u(j)^2  } \|v-v^*\|\notag \\ =  \sqrt{\|u\|^2-u(1)^2} \|v-v^*\| \le \sqrt{\|u\|^2-u(1)^2} D.\label{cb2}
	\end{align}  
	Combine with the above  to get   $-u(1) \frac{\DD_v}{\|a\|} \le \sqrt{\|u\|^2-u(1)^2} D$. Since $u(1) \le 0$ both sides are nonnegative and we can take squares and simplify to get $ \|u\|^2 - u(1)^2 \ge  u(1)^2 Q$ for $Q =  \frac{ \DD_v^2}{D^2\|a\|^2  }$. Add and subtract the same term to the right-hand-side to get $\|u\|^2 - u(1)^2 \ge  -\big(\|u\|^2 - u(1)^2 \big) Q+ \|u\|^2  Q$. Gather factors of $ \|u\|^2 - u(1)^2$ to get   $\big ( \|u\|^2 - u(1)^2 \big) \left (1 + Q \right ) \ge \|u\|^2 Q$ and so
	 
	\begin{align*} \frac{ \|u\|^2 - u(1)^2}{\|u\|^2} \ge \frac{Q}{1+Q} = \frac{1}{1+1/Q}=  \frac{1}{1+D^2\|a\|^2/\DD_v^2} 
	\end{align*} 
	To prove \eqref{conebound1} write $\frac{2 a \cdot u}{\|a\|\|u\|}  =  \left \| \frac{a}{\|a\|} + \frac{u}{\|u\|}\right \|^2-2$ and use how  $ \frac{a}{ \|a\|} =(1,0,\ldots,0)$ to write the right-hand-side componentwise and get
	\begin{align*} 
	\frac{2 a \cdot u}{\|a\|\|u\|}   = \left( 1 + \frac{u(1)}{\|u\|}\right )^2 + \frac{u(2)^2 + \ldots + u(d)^2}{\|u\|^2}-2   
	\ge   \frac{u(2)^2 + \ldots + u(d)^2}{\|u\|^2}-2  \\ =     \frac{\|u\|^2 - u(1)^2  }{\|u\|^2}-2  
	\ge   \frac{1}{1+D^2\|a\|^2/\DD_v^2} -2  
	\end{align*} 
	Divide both sides by $1/2$ to prove  \eqref{conebound1}. 
\end{proof}

\begin{proof}By performing a rotation we can assume $a = (\|a\|,0,\ldots,0)$  .  Then we have $\DD_v = a \cdot ( v-v^*)   = \|a\|(v(1) -v^*(1))$ and so $v(1) -v^*(1) =\DD_v/\|a\|$. For each normal $u \in N_\P(v)$ we know $\P$ is contained in the half-space \mbox{$\{x \in \RR^d: u \cdot  x  \le u \cdot v \}$}. Hence  for each $v^* \in \V^*$ we have $u \cdot v^* \le u \cdot v$. Expand the inequality to get
	\begin{align*} 
	u(1)(v(1)^*-v(1)) \le   u(2) (v(2)-v(2)^*) + \ldots + u(d) (v(d)-v(d)^*)\\
	\implies -u(1) \frac{\DD_v}{\|a\|} \le   u(2) (v(2)-v(2)^*) + \ldots + u(d) (v(d)-v(d)^*) 
	\end{align*} 
	The right-hand-side is the product of two $(d-1)$-dimensional vectors.	By Cauchy Schwarz it is at most $ \sqrt{u(2)^2 + \ldots + u(d)^2}   \sqrt{(v(2)-v^*(2))^2 + \ldots + (v(d)-v^*(d))^2} \le  \sqrt{\|u\|^2-u(1)^2} D$.  Hence we have  $-u(1) \frac{\DD_v}{\|a\|} \le \sqrt{\|u\|^2-u(1)^2} D$. 
	
	First assume $u(1) \le 0$. Since both sides are nonnegative we can take squares and simplify to get
	\begin{align} \label{thh2}\frac{u(1)^2}{\|u\|^2}  \le \frac{1}{1+\DD_v^2/D^2\|a\|^2} 
	\end{align} Now recall $\frac{a}{\|a\|} =(1,0,\ldots,0)$ and write 
	\begin{align} \label{thh3}
	\frac{2a \cdot u}{\|a\|\|u\|}  =  \left \| \frac{a}{\|a\|} + \frac{u}{\|u\|}\right \|^2-2 = \left( 1 + \frac{u(1)}{\|u\|}\right )^2 + \frac{u(2)^2 + \ldots + u(d)^2}{\|u\|^2}-2  
	\end{align} 
	The first term is nonnegative. For the second term write 
	\begin{align} \label{thh4}\frac{ u(2)^2 \ldots + u(d)^2}{\|u\|^2}=\frac{ \|u\|^2  - u(1)^2}{\|u\|^2}= 1- \frac{  u(1)^2}{\|u\|^2} \ge 1 -\frac{1}{1+\DD_v^2/D^2\|a\|^2} =   \frac{1}{1+D^2\|a\|^2/\DD_v^2} 
	\end{align} 
	where we have used (\ref{thh2}) for the inequality. Combining (\ref{thh3}) and (\ref{thh4}) we have
	\begin{align*} 
	\frac{a \cdot u}{\|a\|\|u\|}\ge \frac{1/2}{1+D^2\|a\|^2/\DD_v^2}-1.
	\end{align*} 
	Now assume $u(1) \ge 0$. Then the left-hand-side of the above is nonnegative. Since the right-hand-side is negative the above holds. Hence it holds for all $u \in N_\P(v)$ and the result follows. 
\end{proof}

\section*{Appendix D: Telescoping Sum }

Here we simplify a sum that occurs midway through our analysis.
 
\begin{lemma} \normalfont\label{telescope} Suppose $0 < c_1 \le c_2 \le \ldots \le c_U$ for some $U \in \NN$. Then we have 
	\begin{align} \frac{c_2-c_1}{c_2^2} + \ldots + \frac{c_U-c_{U-1}}{c_{U}^2} \le \frac{1}{c_1} \label{telescopelemma}
	\end{align} 
	\end{lemma}

\begin{proof}For any differentiable function $f: \RR_+ \to \RR$ and $a,b \in \RR_+$ the fundamental theorem of calculus says  $f(b)- f(a) =   \int_a^b f'(t)dt$.  In particular  \begin{align*}f(c_U) -f(c_1)=  \int_{c_1}^{c_U} f'(t)dt =   \int_{c_1}^{c_2} f'(t)dt + \ldots  + \int_{c_{U-1}}^{c_U} f'(t)dt.                                                          \end{align*} Moreover if $f'(t)$ is increasing we have $f'(t) \le f'(c_i)$ over each $[c_{i-1},c_{i}]$ and the above gives
  \begin{align*}f(c_U) -f(c_1)&\le \int_{c_1}^{c_2} f'(c_2)dt + \ldots  + \int_{c_{U-1}}^{c_U} f'(c_U)dt \\
&=   (c_2-c_1) f'(c_2)   + \ldots +  (c_{U} -c_{U-1})f'(c_{U}).                                                          \end{align*} 
Rearrange to get 
  \begin{align*}-(c_2-c_1) f'(c_2)   - \ldots -  (c_{U} -c_{U-1})f'({U-1}) \le f(c_1) -f(c_U).                                                          \end{align*}
  To obtain \eqref{telescopelemma} let $f(t) = 1/t$ in the above. Then $f'(t) = -1/t^2$ is increasing. The left-hand-side of the above becomes the same as \eqref{telescopelemma}. To complete the proof neglect the negative $-f(c_U)$ term on the right-hand-side.
	\end{proof}

\section*{Appendix E: Worst-Case Regret for Lazy Anytime Gradient Descent}

\noindent Here we give the proof that Online Gradient Descent with suitable parameter   has $O \big (L\sqrt N \big)$ regret. The proof uses the   techniques from \cite{Purple1} modified to not mention the time horizon. \\

\noindent{\bf Theorem \ref{worstcase}} {Given cost vectors $b_1.b_2,\ldots, b_N$ with all $\|b_i\| \le L$ Algorithm 1 with parameter $\eta$ has regret satisfying 
	\begin{align*}
	\sum_{i=1}^N b_i \cdot (x_{i} - y^*) \le  LD+ \left (\frac{\|\X\|^2 }{2 \eta} +  2\eta L^2 \right)  \sqrt {N} \end{align*} 
	for $\|\X\|=  \max\{\|x -y_1\|: x \in \X\}$ and $D = \max\{\|x-y\|: x,y \in \X\}$ the diameter of $\X$. In particular for $y_1 \in \X$ and $\eta = \|\X\|/2L$ we have
	\begin{align*}
	\sum_{i=1}^N b_i \cdot (x_{i} - y^*) \le    LD+ 2L \|\X\|  \sqrt N \le 3LD \sqrt N.\end{align*}  }

 \begin{proof}For $n >1$  define the functions $R_{n}(x) = \frac{\sqrt {n-1}}{2 \eta} \|x\|^2$. First we claim each $\ds x_{n} = \Pi_\X(y_n) = \arg \! \min_{x \in \X} \|x-y_n\|^2$ is the   minimiser of $\sum_{i=1}^{n-1} b_i \cdot x+ R_n(x)  $. To that end write $\|x-y_n\|^2  = \|x\|^2 -2 y_n \cdot x + \|y_n\|^2$. We can neglect the  constant term without changing the minimiser. Expand the definition of $y_n$ to see $x_n$ is the minimiser of
 \begin{align} \|x\|^2 -2 y_n \cdot x  =  \|x\|^2 + \frac{2 \eta}{\sqrt {n-1}}  \sum_{i=1}^{n-1} b_i \cdot x  \label{gather}
 \end{align}Likewise multiplying the above by $\frac{\sqrt {n-1}}{2 \eta}$ does not change the minimiser and gives  $R_n(x)$.  Now define the functions
 	$$Q_2(x)   = R_2(x) +  b_1 \cdot x +b_2 \cdot x \qquad Q_{n}(x) = R_n(x) - R_{n-1}(x) + b_n \cdot x \qquad \text{ for } n >2.$$  The above telescope to give  $\sum_{i=2}^n Q_i  = \sum_{i=1}^n b_i  \cdot x + R_{n }(x)$. % and $x_{n+1}$ is the minimiser of $\Sum (Q_i - b_i)$. Hence 
 	The Be the Leader lemma (\cite{CBGames} lemma 3.1) says
 	$ \sum_{i=2}^N  Q_i  (z_i) \le \sum_{i=2}^N  Q_i  (y^*)$
 	for any $\ds z_n \in \arg\! \min_{x \in \X}   Q_i$   and $y^* \in \X$. Expand both sides and gather terms to get
 	\begin{align*}
 	 b_1  \cdot z_2 + \sum_{i=2}^{N} b_i \cdot z_i +  \frac{1}{2 \eta }\sum_{i=2}^{N} (\sqrt {n-1} - \sqrt {n-2})\|z_i\|^2 \le\sum_{i=1}^{N} b_i \cdot y^* +  \frac{\sqrt N}{2 \eta }\|y^*\|^2.
 	\end{align*}
 	Since the second sum on the left-hand-side is nonnegative we can neglect it. Bring the $y^*$ terms to the left and use $\|y^*\| \le \|\X\|$ to get
 	\begin{align*}
 	 b_1  \cdot (z_2-y^*) + \sum_{i=2}^{N} b_i \cdot (z_i-y^*)  \le    \frac{\sqrt N}{2 \eta }\|\X \|^2.
 	\end{align*} To get regret on the left-hand-side add $b_1 \cdot (x_1-z_2)+  \sum_{i=2}^N    b_i \cdot (x_i - z_i)  $ to both sides to get
 	\begin{align}
 	 \sum_{i=1}^N    b_i \cdot (x_i - y^*)   \le  \frac{\sqrt {N}}{2 \eta}\|\X\|^2 + b_1 \cdot (x_1-z_2)+ \sum_{i=2}^N    b_i \cdot (x_i - z_i)\notag \\   \le  \frac{\sqrt {N}}{2 \eta}\|\X\|^2 + LD + \sum_{i=2}^N    \|b_i\| \|x_i - z_i\|  \label{ApA}
 	\end{align}
 	
 	%By Cauchy-Schwarz the last term is at most $ \|b_1\|\|y^*\| \le L \|\X\|$. 
 	%for $D$ the diameter of $\X$. 	
 	To bound the sum on the right we claim $z_n = \Pi_\X \left (- \frac{\eta}{\sqrt {n-1}}   \Sum b_i  \right)$. The proof is similar to how we showed $x_n = \Pi_\X(y_n)$ minimises  $\sum_{i=1}^{n-1} b_i \cdot x + R_{n}(x) $ in the first paragraph. Since the projection is nonexpansice by Theorem 23 of \cite{NonExpansive}) we have 
 	
 	\begin{align*}\|x_n- z_n\|  &= \left \| \Pi_\X \left (- \frac{\eta}{\sqrt {n-1}}   \Sum b_i  \right) - \Pi_\X \left (- \frac{\eta}{\sqrt {n-1}}   \sum_{i=1}^{n-1} b_i  \right)\right \|\\  
 	& \le   \left \|   \frac{\eta}{\sqrt {n-1}}   \Sum b_i      - \frac{\eta}{\sqrt {n-1}}   \sum_{i=1}^{n-1} b_i \right \| = \frac{\eta}{\sqrt {n-1}} \|b_n\|\le \frac{\eta L}{\sqrt {n-1}}  
 	\end{align*} 
 	and the sum in  (\ref{ApA}) is at most
 	\begin{align*}   \sum_{i=2}^N    \|b_i\|\| x_i - z_i\|  \le \eta L^2 \! \sum_{i=2}^N    \frac{1}{\sqrt {i-1}} =   \eta L^2 \! \sum_{i=1}^{N-1}   \frac{1}{\sqrt {i }}     \le   \eta L^2 \int_{0}^{N} \frac{dx}{\sqrt x} = 2\eta L^2 \sqrt N \end{align*}
 	where we can bound the sum by the integral since it is decreasing. Hence (\ref{ApA}) simplifies to
 	\begin{align}\label{ApA1}
 	 \sum_{i=1}^N    b_i \cdot (x_i - y^*)   &\le    LD + \frac{ \|\X\|^2}{2 \eta} \sqrt {N} +   2 \eta L^2  \sqrt N   .%\\
 	  %& \le   \frac{\sqrt {n}}{2 \eta} +  \Sum L  \|y_n - x_n\|
 	\end{align}This proves the first inequality in the theorem statement. For the second set $\eta = \|\X\|/2L$. The last two coefficients simplify to $L \|X\|$  and the right-hand-side becomes  $LD + 2 L\|\X\| \sqrt N \le 3 LD \sqrt N$ since $\|\X\| \le D$ and $1 \le \sqrt N$. 
\end{proof} 

\end{document}